\documentclass{article}

\usepackage[final]{neurips_2021}

\usepackage[utf8]{inputenc} % allow utf-8 input
\usepackage[T1]{fontenc}    % use 8-bit T1 fonts

\usepackage{microtype}
\usepackage{graphicx}
\usepackage{subfigure}

\usepackage{hyperref}

\usepackage{algorithm}
\usepackage{algpseudocode}
\usepackage{times}
\usepackage{latexsym}
\usepackage{booktabs}
\usepackage{graphicx}
\usepackage{amsmath}
\usepackage{multirow}
\usepackage{wrapfig}
\usepackage{hhline}
\usepackage{tabularx}
\usepackage{xcolor}
\usepackage{makecell}
\usepackage{comment}

\include{Definitions}

\begin{document}

\title{Bridging the Gap Between Practice and PAC-Bayes Theory in Few-Shot Meta-Learning}
%Bridging The Gap Between Practical Meta-learning and PAC-Bayes Theory in the Few-shot Setting
% Meta-learning in the few shot setting: bridging the gap between practice and PAC-Bayes theory
% Meta-learning in the few shot setting: bridging the gap between practice and PAC-Bayes theory
% PAC-Bayesian bounds for Practical Meta-learning in the Few-shot Setting
% Practical PAC-Bayesian Meta-learning in Few-shot
%Bridging the gap between PAC-Bayesian bounds and Practical Meta-learning in the Few-shot Setting

% PAC-Bayesian bounds for practical Meta-learning in the Few-shot Setting

\author{
  Nan Ding \\
  Google Research \\
  %Venice, CA 90291 \\
  \texttt{dingnan@google.com}
  \And
  Xi Chen\\
  Google Research \\
  %Cambridge, MA 02142 \\
  \texttt{chillxichen@google.com}
  \And
  Tomer Levinboim\\
  Google Research \\
  %Venice, CA 90291 \\
  \texttt{tomerl@google.com}
  \And
  Sebastian Goodman\\
  Google Research \\
  %Venice, CA 90291 \\
  \texttt{seabass@google.com}
  \And
  Radu Soricut \\
  Google Research \\
  %Venice, CA 90291 \\
  \texttt{rsoricut@google.com}
}

\maketitle

\begin{abstract}
%Meta-learning is a machine learning approach that suggests to pretrain models on existing observed tasks and then adapt them to new (unseen) target tasks, possibly with limited data.
Despite recent advances in its theoretical understanding, there still remains a significant gap in the ability of existing PAC-Bayesian theories on meta-learning to explain performance improvements in the few-shot learning setting, where the number of training examples in the target tasks is severely limited. 
This gap originates from an assumption in the existing theories which supposes that the number of training examples in the observed tasks and the number of training examples in the target tasks follow the same distribution, an assumption that rarely holds in practice.
By relaxing this assumption, we develop two PAC-Bayesian bounds tailored for the few-shot learning setting and show that two existing meta-learning algorithms (MAML and Reptile) can be derived from our bounds, thereby bridging the gap between practice and PAC-Bayesian theories. 
Furthermore, we derive a new computationally-efficient PACMAML algorithm, and show it outperforms existing meta-learning algorithms on several few-shot benchmark datasets.
\end{abstract}

\section{Introduction}

Recent advances in machine learning and neural networks have resulted in effective but parameter-bloated, data-hungry models. When the training data for a target task of interest is insufficient, such overparameterized models may easily overfit to the training data and exhibit poor generalization abilities. To address this problem, several research efforts  have focused on designing a learning strategy that can leverage the training data of other tasks for the sake of improving the performance of some specific target task(s). Specifically, in the meta-learning (also called learning-to-learn or lifelong-learning) setting~\citep{baxter1998theoretical,Ravi2017OptimizationAA}, a meta-learner first extracts knowledge from a set of observed (meta-training) tasks and subsequently, this knowledge enables a base-learner to better adapt to the new, possibly data-limited target (meta-testing) task. 
%The meta-learning framework encapsulates two related and well-known learning paradigms: pretraining-finetuning~\citep{russakovsky2015imagenet,devlin2019bert} and model agnostic meta-learning (MAML)~\citep{finn2017model}, both of which have made significant practical impact on computer vision~\citep{russakovsky2015imagenet}, language understanding~\citep{devlin2019bert}, reinforcement learning~\citep{finn2017model} and many other research fields.
The meta-learning framework has been successfully applied and made significant practical impact on computer vision~\citep{russakovsky2015imagenet}, language understanding~\citep{devlin2019bert}, reinforcement learning~\citep{finn2017model} and many other research fields.

%Despite its great empirical success, the theoretical understanding of meta-learning is still lacking. In particular, how does meta-learning utilizes its knowledge coming from the observed task data and generalizes to the unseen target task that is data-limited. 
In parallel to its impressive empirical success, a series of theoretical works~\citep{tripuraneni2020theory,pentina14,amit18a,rothfuss2020pacoh}  study how meta-learning utilizes the knowledge obtained from the observed task data and how it generalizes to the unseen target task. 
Among the generalization bounds, PAC-Bayes bounds~\citep{mcallester1999some,germain2009pac} are considered especially tight and have already been proposed for meta-learning~\citep{pentina14,amit18a,rothfuss2020pacoh}. However, there still remains a gap between these existing PAC-Bayesian bounds and their practical application (especially in the few-shot setting), which originates from the assumption that 
%the number of training examples $\tilde{m}$ in the observed tasks and the number of training examples $m$ in the target task follow the same distribution (i.e. $\tilde{m} \simeq m$).
the observed task environment $\tilde{T}$ and the target task environment $T$ are the same. In the PAC-Bayesian meta-learning setting, a task environment $T$ is a distribution from which $(D, m)$ is drawn from, where $D$ is the data distribution and $m$ is the number of training examples for the task. Although there is research work studying the case of general environment change (e.g.~\citep{pentina2015lifelong}) or data domain change (e.g.~\cite{germain2016new}), to the best of our knowledge, there is little work focusing on the case where only the number of training examples $\tilde{m}$ in the observed tasks and $m$ in the target task do not follow the same distribution.
%This assumption does not fit well in practice, where 
In practice, such mismatch commonly happens, because there is usually significantly more data in observed tasks than the target tasks, especially in the few-shot case.
Without explicitly addressing this mismatch, the scope of the current theory is severely limited, and it prohibits a useful analysis on practical meta-learning algorithms such as MAML~\cite{finn2017model}. For example, when the number of training examples $m$ in the target task is small, the existing bounds yield a large generalization gap which grows with $O(1/m)$. 
% TODO(tomerl): keep the below abstract, but what is the difference between thm3 and thm4? (Done)
% Maybe mention the task env' around here
In this paper, we bring the theory closer to practice by studying the setting where there are significantly more training examples in the observed task than in the target task (i.e., $\tilde{m} \gg m$). 
In Section \ref{sec:meta_training_relaxed_1}, we study two practical meta-training strategies and provide their PAC-Bayesian bounds in Theorem~\ref{thm:pac-meta2} and Theorem~\ref{thm:pac-meta3}. Both results are able to bring down the scaling coefficient of the bound from $O(1/m)$ to $O(1/\tilde{m})$. However, Theorem~\ref{thm:pac-meta2} introduces a penalty term in the bound that captures the discrepancy between the observed and target task environment.
Motivated by MAML~\cite{finn2017model}, we show with Theorem~\ref{thm:pac-meta3} that we can eliminate the penalty term by utilizing a subsampling strategy, yielding a much tighter bound. 

This theoretical work also bridges the gap from practice to theory, as we further show that the maximum-a-posteriori (MAP) estimates of our bounds (in which the base-learner and the hyper-posterior are both approximated by Dirac-measures) yield various popular meta-learning algorithms, including multi-task pretraining~\citep{russakovsky2015imagenet}, Reptile~\citep{nichol1803first} and MAML~\citep{finn2017model}. 
In that sense, our PAC-Bayesian theories provide a different perspective for understanding and justifying these commonly used algorithms (Section \ref{sec:relation}).

Lastly, in Section \ref{sec:instantiation}, we propose PACMAML, a novel PAC-Bayesian meta-learning algorithm based on Theorem \ref{thm:pac-meta3}. As opposed to MAML, our algorithm does not have higher-order derivatives in the gradient, and therefore represents a significant improvement in computational efficiency.
In Section \ref{sec:experiments}, we conduct numerical experiments that empirically support the correctness of our theorems, and report the effectiveness of the new PACMAML algorithm, which obtains superior results on several few-shot benchmark datasets.

\section{Preliminaries}
We begin by reviewing the background and settings of the existing PAC-Bayesian bounds for meta-learning. Our notation mainly follows that of \citep{rothfuss2020pacoh}, which is itself adapted from \citep{pentina14,amit18a,baxter1998theoretical}.
% Note(tomerl): It might be nice to set reader's expectations here about the meaning of these bounds.

\paragraph{PAC-Bayesian for Supervised Learning}
In supervised learning, a learning task is characterized by a data distribution $D$ over a data domain $Z$ where every example $z = (x, y)$. A hypothesis $h$ from the hypothesis space $H$ allows us to make predictions based on inputs $x$. The quality of the predictions is measured by a loss function $l(h, z)$, where the goal is to minimize the expected loss 
$L(h, D) = \EE_{z \sim D} l(h, z)$. 
Typically, $D$ is unknown and instead we are given a set of $m$ observations 
$S \sim D^m = \{z_i \sim D\}_{i=1}^m$, in which case the empirical error on $S$ is simply
$\hat{L}(h, S) = \frac{1}{m} \sum_{i=1}^m l(h, z_i)$. 

In the PAC-Bayesian setting,
we assume that the learner has prior knowledge of the hypothesis space $H$ in the form of a prior distribution $P(h)$. When the learner observes a training dataset $S$, it updates the prior into a posterior distribution $Q$. 
We formalize such a \emph{base learner} $Q(S, P)$ that takes a dataset and a prior as input and outputs a posterior. 

The expected error of the posterior $Q$ is called the Gibbs error $L(Q, D) = \EE_{h \sim Q} L(h, D)$, and its empirical counterpart is $\hat{L}(Q, S) = \EE_{h \sim Q} \hat{L}(h, S)$.
The PAC-Bayesian framework provides the following bound over $L(Q, D)$ based on its empirical estimate $\hat{L}(Q, S)$.

\begin{theorem} [\citep{alquier2016properties,germain2009pac}]
\label{thm:pac-super}
Given a data distribution $D$, a hypothesis space $H$, a prior $P$, a confidence level $\delta \in (0, 1]$, and $\beta > 0$, with probability at least $1 - \delta$ over samples $S \sim D^m$, we have for all posterior $Q$,
{\small\begin{align}
    &L(Q,D)
    \le \hat{L}(Q, S) + \frac{1}{\beta} \rbr{D_{KL}(Q \| P) + \log\frac{1}{\delta} } + \frac{m}{\beta}\Psi(\frac{\beta}{m}) \label{eq:super}
\end{align}}
where $\Psi(\beta) = \log \EE_{h \sim P} \EE_{z \sim D} \exp(\beta(l(h,z) - L(h, D)))$.
\end{theorem}

\paragraph{PAC-Bayesian for Meta-Learning}
In the meta-learning setting, the meta-learner observes different tasks $\tau_i = (D_i, m_i$) during the meta-training stage, where all tasks share the same data domain $Z$, hypothesis space $H$ and loss function $l(h, z)$. For each observed task $\tau_i$, the meta-learner observes a training set $S_i$ of size $m_i$ which is assumed to be sampled i.i.d. from its respective data distribution $D_i$ (that is, $S_i \in D_i^{m_i}$). 
We further assume that each task $\tau_i = (D_i, m_i)$ is drawn i.i.d. from an environment $T$, which itself is a probability distribution over the data distributions and the sample sizes. 
The goal of meta-learning is to extract knowledge from the observed tasks $\tau_i$, which can then be used as prior knowledge for learning on new (yet unobserved) target tasks $\tau = (D, m) \sim T$. 
This prior knowledge is represented as a prior distribution $P(h)$ over learning hypotheses $h$, and it is subsequently used by the base learner $Q(S, P)$ for inference over the target tasks.

In the meta-learning PAC-Bayes framework, the meta-learner presumes a hyper-prior $\Pcal(P)$ as a distribution over priors $P$. 
Upon observing datasets $S_1, \ldots, S_n$ from multiple tasks, the meta-learner updates the hyper-prior to a hyper-posterior $\Qcal(P)$. 
The performance of this hyper-posterior, also called the transfer-error, is measured as the expected Gibbs error when sampling priors $P$ from $\Qcal$ and applying the base learner:
\begin{align}
R(\Qcal, T) &:= \EE_{P \sim \Qcal} \EE_{(D,m) \sim T} \EE_{S \sim D^m} \sbr{L(Q(S, P), D)}.\label{eq:ge}
\end{align}
While $R(\Qcal, T)$ is unknown in practice, it can be estimated using the empirical error,
\begin{align}
    \hat{R}(\Qcal, S_{i=1}^n) := \EE_{P \sim \Qcal} \sbr{\frac{1}{n} \sum_{i=1}^n \hat{L}(Q(S_i, P), S_i)}. \label{eq:er}
\end{align}
In \citep{pentina14,rothfuss2020pacoh}, the following PAC-Bayesian meta-learning bound is provided:
\begin{theorem}[\citep{pentina14,rothfuss2020pacoh}]
\label{thm:pac-meta1-new}
Given a task environment $T$ and a set of $n$ observed tasks $(D_i, m_i) \sim T$, let $\Pcal$ be a fixed hyper-prior and $\lambda > 0$, $\beta > 0$, with probability at least $1-\delta$ over samples $S_1 \in D_1^{m_1}, \ldots, S_n \in D_n^{m_n}$, we have, for all base learner $Q$ and all hyper-posterior $\Qcal$,
{\small\begin{align}
    R(\Qcal, T) \le &\hat{R}(\Qcal, S_{i=1}^n) + \rbr{\frac{1}{\lambda} + \frac{1}{n\beta}} D_{KL}(\Qcal \| \Pcal)\nonumber\\
    &+\frac{1}{n\beta}\sum_{i=1}^n \EE_{P \sim \Qcal} \sbr{D_{KL}(Q(S_i, P) \| P)}+ C(\delta, \lambda, \beta, n, m_i).  \label{eq:task_union1_new}
\end{align}}
\end{theorem}
Here $C(\delta, \lambda, \beta, n, m_i)$ contains $\Psi$ and $\frac{1}{\delta}$ terms as in Eq.\eqref{eq:super} (see Appendix \ref{sec:proof:thm5}), and can be bounded by a function that is independent of $\Qcal$ for both bounded and unbounded loss functions under moment constraints (see details in \citep{rothfuss2020pacoh}).
From a Bayesian perspective, meta-learning attempts to learn a good hyper-posterior $\Qcal$ such that for all tasks in the task environment $T$, the divergence terms ${D_{KL}(Q(S_i, P) \| P)}$ would be substantially smaller in expectation when $P \sim \Qcal$ compared to when $P \sim \Pcal$, such as in the ordinary supervised learning setting of Eq.\eqref{eq:super}.

The hyperparameters $\lambda$ and $\beta$ can be adjusted to balance between the first three terms of the bound and the $C$ function. Defining the harmonic mean of $m_i$ as $\tilde{m} = (\sum_{i=1}^n 1/n m_i)^{-1}$, a common choice is $\lambda \propto n$ and $\beta \propto \tilde{m}$\footnote{Another common choice is $\lambda\propto \sqrt{n}$ and $\beta\propto \sqrt{\tilde{m}}$, so that the bound is asymptotically consistent, and scales with $O(\frac{1}{\sqrt{\tilde{m}}})$. 
However, in practice the bound with $\beta \propto \tilde{m}$ is usually tighter~\citep{germain2016pac}.}.
In this case, the generalization gap $R(\Qcal, T) - \hat{R}(\Qcal, S_{i=1}^n)$ becomes at least $O(\frac{1}{\tilde{m}})$ (from the 3rd-term on the RHS of Eq.\ref{eq:task_union1_new}).
In the next section, we examine an assumption in this bound which makes it impractical for the few-shot setting.

\section{Bridging the Gap between Practice \& Theory of Few-Shot Meta-Learning}
\label{sec:meta_training_relaxed}
The previous PAC-Bayesian meta-learning bound (Theorem~\ref{thm:pac-meta1-new}) assumes that the number of training examples $m_i$ for the observed tasks $\tau_i$ and the number of training examples $m$ for the target task $\tau$ are drawn from the same distribution (i.e. $\EE_T[m_i]=\EE_T[m]$).
However, practical applications of meta-learning such as \citep{russakovsky2015imagenet,devlin2019bert} operate in a setting where there are far more training examples in the observed tasks than in the target task.
Moreover, focusing on the few-shot setting (where $m$ is particularly small) exposes a gap between theory and practice -- Theorem~\ref{thm:pac-meta1-new} is unable to use the large number of observed samples and can only produce a loose bound of $O(\frac{1}{m})$ which is ineffective at explaining the impressive generalization performance of meta-learning as reported in practice.

In this section we attempt to close this gap by deriving an effective PAC-Bayesian bound (Theorem~\ref{thm:pac-meta3}) tailored for the few-shot setting. 
Interestingly, the bounds derived in this section also provide PAC-Bayesian justifications for two practical algorithms, Reptile and MAML.

\subsection{Practical PAC-Bayesian Bounds for Few-Shot Meta-Learning}
\label{sec:meta_training_relaxed_1}
A first attempt at leveraging the larger number of examples $m_i$ in the observed tasks is to directly
%apply the same learning strategy as in Theorem~\ref{thm:pac-meta1-new},
follow the learning strategy of Theorem~\ref{thm:pac-meta1-new},
%(albeit a shift of task environment between the observed $\tilde{T}$ and the target $T$)
by bounding $R(\Qcal, T)$ using the empirical risk $\hat{R}(\Qcal, S_{i=1}^n)$,
%all available data $S_i \in D_i^{m_i}$ in observed task $\tau_i$ is used to train its base-learner $Q(S_i, P)$ and then the empirical error of the same data $S_i$ is used to bound the transfer error $R(\Qcal, T)$. 
with $S_i \in D_i^{m_i}$ and $(D_i, m_i) \sim \tilde{T}$, despite the change of task environment from $T$ to $\tilde{T}$. 
This slight generalization leads to the following bound (with proof in Appendix \ref{sec:proof:thm7}):
\begin{theorem}
\label{thm:pac-meta2}
For a target task environment $T$ and an observed task environment $\tilde{T}$ where $\EE_{\tilde{T}}[D]=\EE_T[D]$ and $\EE_{\tilde{T}}[m]\ge\EE_T[m]$, let $\Pcal$ be a fixed hyper-prior and $\lambda > 0$, $\beta > 0$, then with probability at least $1-\delta$ over samples $S_1 \in D_1^{m_1}, \ldots, S_n \in D_n^{m_n}$ where $(D_i, m_i) \sim \tilde{T}$, we have, for all base learners $Q$ and hyper-posterior $\Qcal$, 
{\small \begin{align}
    R(\Qcal, T) \le &\hat{R}(\Qcal, S_{i=1}^n) + \rbr{\frac{1}{\lambda} + \frac{1}{n\beta}} D_{KL}(\Qcal \| \Pcal)\nonumber\\
    &+\frac{1}{n\beta}\sum_{i=1}^n \EE_{P \sim \Qcal} \sbr{D_{KL}(Q(S_i, P) \| P)}
    + C(\delta, \lambda, \beta, n, m_i)
    + \Delta_{\lambda}(\Pcal, T, \tilde{T}),  \label{eq:task_union2}
\end{align}}
where $\Delta_{\lambda}(\Pcal, T, \tilde{T})=\frac{1}{\lambda}\log \EE_{P \in \Pcal} e^{\lambda(R(P, T) - R(P, \tilde{T}))}$.
\end{theorem}
When $\EE_{\tilde{T}}[m_i] \gg \EE_{T}[m]$, this decoupling of the task environments seems beneficial at first, because $O(\frac{1}{\tilde{m}})$ is smaller compared to Eq.\eqref{eq:task_union1_new} when $\beta \propto \tilde{m}$. 
Unfortunately however, Eq.\eqref{eq:task_union2} introduces an additional 
penalty term $\Delta_{\lambda}$, which increases as $\EE_{\tilde{T}}[\tilde{m}]$ gets larger.

To understand the influence of $\Delta_{\lambda}$, we plot the (blue) bound of Eq.\eqref{eq:task_union2} in Fig.\ref{fig:sinusoid_bound} by using the synthetic Sinusoid regression task (see details in Section \ref{sec:regression} and in Appendix \ref{appsubsec:bound}) where we fixed $m=5$ and varied $m_i$ from 5 to 100.
When $m_i=m=5$, Eq.\eqref{eq:task_union2} reduces to Eq.\eqref{eq:task_union1_new} and $\Delta_{\lambda}=0$. Contrary to intuition, increasing $m_i$ does not reduce the bound, but instead makes it worse due to the rapid increase of $\Delta_\lambda$.
\begin{wrapfigure}{R}{5cm}\vspace{-10pt}
\centering
    \includegraphics[width=0.33\textwidth]{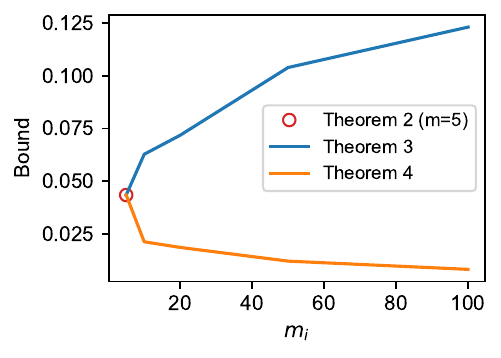}
    \caption{The PAC-Bayesian bounds of Theorems \ref{thm:pac-meta1-new}, \ref{thm:pac-meta2}, \& \ref{thm:pac-meta3} as evaluated over the Sinusoid dataset. Some constant terms are neglected (see Appendix \ref{appsubsec:bound} for more details).}
    \label{fig:sinusoid_bound}
\end{wrapfigure}
% TODO(dingnan): i feel a transition sentence is needed here, such as: Since $\Delta_\lambda$ is very detrimental, ... (Done)

Can we utilize more training examples without introducing a penalty term such as $\Delta_\lambda$?
In the definition of $\hat{R}(\Qcal, S_{i=1}^n)$ (Eq.\eqref{eq:er}), we note that the training dataset $S_i$ of the observed task $\tau_i$ is used twice: first in training the base-learner $Q(S_i, P)$, and then, in evaluating the empirical risk $\hat{L}(Q, S_i)$.
In analyzing the proof of the theorem (see Appendix~\ref{sec:proof:thm7}), it can be seen that the penalty term arises exactly because $Q(S_i, P)$ is trained over more samples compared to $Q(S, P)$ of the target task, which results in the more powerful base-learners during meta-training than the one for the target task.
%, due to the discrepancy between $\tilde{T}$ and $T$ which makes $\EE_{\tilde{T}}[m_i] > \EE_T[m]$ (where $|S_i|=m_i, |S|=m$).

This motivates us to develop a MAML-inspired learning strategy, in which we maintain the same target-task training environment $T$ for the base-learners of the observed tasks: we first sample a subset $S_i'  \in D_i^{m_i'}$ from $S_i$ where $m'_i$ and $m$ follow the same distribution and $m_i' \le m_i$. Then we use only the subset $S_i'$ to train the base-learner $Q(S_i', P)$. At the same time, all examples of $S_i \in D_i^{m_i}$ are used for evaluating the empirical risk $\hat{L}(Q, S_i)$, so that the larger $m_i$ in the empirical risk $\hat{L}(Q, S_i)$ help tightening the generalization gap.
This new strategy leads to the following bound (proof in Appendix \ref{sec:proof:thm9}):

\begin{theorem}
\label{thm:pac-meta3}
For a target task environment $T$ and an observed task environment $\tilde{T}$ where $\EE_{\tilde{T}}[D]=\EE_T[D]$ and $\EE_{\tilde{T}}[m]\ge\EE_T[m]$, let $\Pcal$ be a fixed hyper-prior and $\lambda > 0$, $\beta > 0$, then with probability at least $1-\delta$ over samples $S_1 \in D_1^{m_1}, \ldots, S_n \in D_n^{m_n}$ where $(D_i, m_i) \sim \tilde{T}$, and subsamples $S_1' \in D_1^{m_1'} \subset S_1, \ldots, S_n' \in D_n^{m_n'} \subset S_n$, where $\EE[m_i'] = \EE_T[m]$, we have, for all base learner $Q$ and all hyper-posterior $\Qcal$,
{\small\begin{align}
    R(\Qcal, T) \le & \EE_{P \sim \Qcal} \sbr{\frac{1}{n}\sum_{i=1}^n \hat{L}(Q(S_i', P), S_i)}
    + \rbr{\frac{1}{\lambda} + \frac{1}{n\beta}} D_{KL}(\Qcal \| \Pcal) \nonumber \\
    &+\frac{1}{n\beta}\sum_{i=1}^n \EE_{P \sim \Qcal} \sbr{D_{KL}(Q(S_i', P) \| P)} 
    + C(\delta, \lambda, \beta, n, m_i).\label{eq:task_union3}
\end{align}}
\end{theorem}
This bound is still $O(\frac{1}{\tilde{m}})$ when choosing $\beta \propto \tilde{m}$, but unlike Eq.\eqref{eq:task_union2}, it does not have an additional penalty term in Eq.\eqref{eq:task_union3}, which is due to the shared training environment $T$ of the base-learners in both observed and target tasks.
Importantly, the resulting bound is effective in the few-shot setting as an increase in the number of observed examples $m_i$ monotonically tightens the generalization gap.
This is visually demonstrated in Fig.\ref{fig:sinusoid_bound} in which the (orange) bound of Eq.\eqref{eq:task_union3} in Theorem~\ref{thm:pac-meta3} is monotonically decreasing as $m_i$ increases, while the bound in Theorem~\ref{thm:pac-meta1-new} is limited only to $m_i=5$ and the bound of Theorem~\ref{thm:pac-meta2} grows.

\subsection{Justifying Reptile and MAML using PAC-Bayesian Bounds}
\label{sec:relation}
It is worth noting that Theorems \ref{thm:pac-meta2} and \ref{thm:pac-meta3} not only address more practical scenarios in which observed (meta-training) examples are more abundant than the target examples, but they also serve as a justification for popular and practical meta-learning algorithms: Reptile~\citep{nichol1803first} and MAML~\citep{finn2017model}.

To show this, let us consider the maximum-a-posteriori (MAP) approximations on the hyper-posterior $\Qcal(P)$ and base-leaner $Q_i(h), \forall i=1,\ldots,n$, with Dirac measures. In addition, we use the isotropic Gaussian priors with variance hyperparameter $\sigma_0^2$ and $\sigma^2$ for the hyper-prior $\Pcal(P)$ and the prior $P(h)$.
The hypothesis $h$ is parameterized by $\vb$.
Then we have
{\small\begin{align*}
    \Pcal(P) &= \Ncal(\pb | 0, \sigma_0^2), \;
    \Qcal(P) = \delta(\pb = \pb_0), \;
    P(h_{\vb}) = \Ncal(\vb | \pb, \sigma^2), \;
    Q_i(h_{\vb}) = \delta(\vb = \qb_i),
\end{align*}}
and the goal of MAP approximation is to find the optimal meta-parameters $\pb_0$. 
With the above assumptions, the PAC-Bayesian bound (denoted PacB) of Eq.\eqref{eq:task_union2} and Eq.\eqref{eq:task_union3} with respect to $\pb_0$ becomes (up to a constant, see Appendix \ref{sec:appendix-dirac}),
\begin{align}
PacB(\pb_0)
=&\frac{1}{n} \sum_{i=1}^n  \hat{L}(\qb_i, S_i) + \frac{\tilde{\xi}\|\pb_0\|^2}{2\sigma_0^2} + \frac{1}{n\beta}\sum_{i=1}^n \frac{\|\pb_0-\qb_i\|^2}{2\sigma^2}, \label{eq:pac_delta1}
\end{align}
where $\tilde{\xi} = \frac{1}{\lambda} + \frac{1}{n\beta}$. 
Here, $\qb_i$ can be any function of $\pb_0$ and $S_i$ for Eq.\eqref{eq:task_union2} (or $\pb_0$ and $S_i'$ for Eq.\eqref{eq:task_union3}), such that the only free variable in Eq.\eqref{eq:pac_delta1} is $\pb_0$. 
Indeed, by setting $\qb_i$ according to the choices below, we can derive the gradients of several meta-learning algorithms.

When $\qb_i = \pb_0$, the gradient of Eq.\eqref{eq:pac_delta1} reduces to that of multi-task pretraining~\citep{russakovsky2015imagenet,devlin2019bert},
\begin{align*}
   \lim_{\qb_i \to \pb_0}\frac{d (PacB)}{d \pb_0} = \frac{\tilde{\xi}\pb_0}{ \sigma_0^2} + \frac{1}{n}\sum_{i=1}^n \frac{d}{d \pb_0} \hat{L}(\pb_0, S_i).
\end{align*}
On the other hand, if we use the optimal Dirac-base-learner $\qb_i^*$ of $\pb_0$ and $S_i$, such that
\begin{align}
    \qb_i^* =& \argmin_{\qb_i} \sbr{\hat{L}(\qb_i, S_i) + \frac{\|\pb_0-\qb_i\|^2}{2\beta\sigma^2}}, \label{eq:optbase_pac_delta1}
\end{align}
then the gradient of Eq.\eqref{eq:pac_delta1} becomes substantially simpler (see details in the Appendix \ref{sec:appendix-dirac}),
\begin{align}
\frac{d (PacB)}{d \pb_0} %&= \frac{\partial (PacB)}{\partial \pb_0} + \inner{\frac{\partial \qb_i^*}{\partial \pb_0}} {\frac{\partial (PacB)}{\partial \qb_i^*}} \nonumber\\
%&= \frac{\partial (PacB)}{\partial \pb_0} \nonumber\\
&= \frac{\tilde{\xi}\pb_0}{ \sigma_0^2} + \frac{1}{n}\sum_{i=1}^n \frac{\pb_0-\qb_i^*}{\beta\sigma^2}, \label{eq:optgrad_pac_delta1}
%\\&=\frac{\pb_0}{\tilde{\xi} \sigma_0^2} + \frac{1}{n}\sum_{i=1}^n \nabla_{\qb^*_i} L(\qb_i^*, S_i). \label{eq:impgrad_pac_delta1}
\end{align}
and in fact, Eq.\eqref{eq:optgrad_pac_delta1} is equivalent to the meta-update rule
of the Reptile algorithm~\citep{nichol1803first}, whose inner-loop is an approximate algorithm for solving the optimal Dirac-base-learner $\qb_i^*$. 

Lastly, when $\qb_i$ is a few gradient descent steps of $\hat{L}(\qb_i, S_i')$ with initial $\qb_i=\pb_0$, the gradient of Eq.\eqref{eq:pac_delta1} reduces to that of the MAML algorithm\footnote{A slight difference is that MAML usually assumes $S_i \cap S_i' = \emptyset$; while in our setting, we assume $S_i' \subset S_i$. However, Theorem \ref{thm:pac-meta3} still holds when $S_i \cap S_i' = \emptyset$. 
}~\citep{finn2017model} as $\sigma^2 \to \infty$,
\begin{align*}
    \lim_{\sigma^2 \to \infty}\frac{d (PacB)}{d \pb_0} = \frac{\tilde{\xi}\pb_0}{ \sigma_0^2}  + \frac{1}{n}\sum_{i=1}^n \frac{d}{d \pb_0} \hat{L}(\qb_i, S_i).
\end{align*}
One observation here is that, since $\qb_i$ is function of the gradient of $\pb_0$, $d \qb_i/d \pb_0$ involves high-order gradient w.r.t. $\pb_0$, which would result in a computationally intensive algorithm. In the next section we present a computationally efficient algorithm which relies only on first-order derivatives.

\section{PAC-Bayesian Meta-Learning Algorithms in the Few-Shot Setting}
\label{sec:instantiation}
In this section we present two PAC-Bayesian based Meta-Learning algorithms with non-Dirac base-learners. 
We first derive their objective functions from the RHS of Eq.\eqref{eq:task_union2} and Eq.\eqref{eq:task_union3}, and then derive low-variance gradient estimators for their optimization.

First, since Eq.\eqref{eq:task_union1_new} and Eq.\eqref{eq:task_union2} only differ by $\Delta_{\lambda}$, we follow \citep{rothfuss2020pacoh} and 
plug in their proposed Gibbs posterior based base-learner $Q^*(S_i, P)(h) = P(h)\exp(-\beta \hat{L}(h, S_i)) / Z_\beta(S_i, P)$ into Eq.\eqref{eq:task_union2}, which minimizes Eq.\eqref{eq:task_union2} w.r.t. $Q$.
This yields that, with at least $1-\delta$ probability, 
{\small\begin{align}
    R(\Qcal, T) \le & \frac{1}{n}\sum_{i=1}^n\EE_{P \sim \Qcal} \underbrace{\sbr{-\frac{1}{\beta}  \log Z_\beta(S_i, P)}}_{W_1}
    + \tilde{\xi} D_{KL}(\Qcal \| \Pcal) + \Delta_{\lambda} + C \label{eq:inst-3}
\end{align}}
where $\tilde{\xi} = \frac{1}{\lambda} + \frac{1}{n\beta}$ and $C$ is the same constant from the previous bounds. %It is worth noting that $\Delta_\lambda$ is the minimum of two functions where the second function is dependent on $\Qcal$ but is inestimable in general. Therefore, we only consider the first function of $\Delta_\lambda$ which is a constant of $\Qcal$ and can be neglected during inference or optimization of $\Qcal$. This yields the same PACOH objective as in~\citep{rothfuss2020pacoh}.
Since $\Delta_\lambda$ is independent of $\Qcal$ and can be neglected during inference or optimization of $\Qcal$, it reduces to the same PACOH objective as in~\citep{rothfuss2020pacoh}.

On the other hand, the same Gibbs posterior cannot be used as the base learner of Eq.\eqref{eq:task_union3}, because the Gibbs posterior would depend on $S_i$, while the base learner in Eq.\eqref{eq:task_union3} should only be dependent on $S_i' \subset S_i$. Therefore, we use the following posterior $Q_i^\alpha$ with hyperparameter $\alpha$,
\begin{align*}
    Q_i^{\alpha}(S_i', P)(h) = \frac{P(h)\exp(-\alpha \hat{L}(h, S_i'))}{Z_\alpha(S_i', P)}.
\end{align*}

Plugging into Eq.\eqref{eq:task_union3} (derivations in Appendix) yields that, with at least $1-\delta$ probability,
{\small\begin{align}
    R(\Qcal, T) \le & \frac{1}{n}\sum_{i=1}^n\EE_{P \sim \Qcal} \underbrace{\sbr{-\frac{1}{\beta}  \log Z_\alpha(S_i', P) + \hat{L}^{\Delta}_{\frac{\alpha}{\beta}}(Q^{\alpha}_i, S_i, S_i')}}_{W_2}
    + \tilde{\xi} D_{KL}(\Qcal \| \Pcal) + C. \label{eq:inst-5}
\end{align}}
where $\hat{L}^{\Delta}_{\frac{\alpha}{\beta}}(Q^{\alpha}_i, S_i, S_i') \triangleq \hat{L}(Q^{\alpha}_i, S_i) - \frac{\alpha}{\beta}\hat{L}(Q_i^{\alpha}, S_i')$.
We refer to the RHS of this equation as the PACMAML objective, because Eq.\eqref{eq:inst-5} comes from the PAC-Bayesian bound of Eq.\eqref{eq:task_union3}, which
is similar to MAML in subsampling the training data for base-learners.

Given these two objectives, the next step is to estimate the gradients of $W_1$ and $W_2$, which can then be plugged into Monte-Carlo methods for estimating a hyper-posterior distribution of $\Qcal$ (or optimization methods for finding an MAP solution).

\paragraph{Gradient Estimation}
In $W_1$ and $W_2$, the terms $Z_{\beta}, Z_{\alpha}, \hat{L}^{\Delta}_{\frac{\alpha}{\beta}}(Q^{\alpha}_i, S_i, S_i')$ all involve integrations over $h$. When $P(h)$ is Gaussian and $\hat{L}(h, S_i)$ is a squared loss, such integrations have closed form solutions and the gradients can be analytically obtained. 
However, when $\hat{L}(h, S_i)$ is not a squared loss (such as the softmax loss), the integration does not have a closed form solution and we resort to approximations.
For example, \cite{rothfuss2020pacoh} directly approximates the objective $W_1$ with Monte-Carlo sampling, which however results in a biased gradient estimator. 

Here, we follow an alternative approach from the REINFORCE algorithm~\citep{Williams92}, which instead approximates \emph{the gradient of the objective} with Monte-Carlo methods, and has the benefit that the resulting gradient estimator is unbiased. 
Assuming that the model hypothesis $h$ is parameterized by $\vb$ such that $\hat{L}(h, S_i) \triangleq \hat{L}(\vb, S_i)$, and $\vb$ has prior $P(\vb) = \Ncal(\vb| \pb, \sigma^2)$ with meta-parameter $\pb$, then
\begin{align*}
\log Z_\beta(S_i, \pb) = \log \int \Ncal(\vb| \pb, \sigma^2)\exp(-\beta \hat{L}(\vb, S_i)) d\vb. %\label{eq:inst-3'}
\end{align*}
Note that $\pb$ appears in the probability distribution $\Ncal(\vb| \pb, \sigma^2)$ of the expectation, and the naive Monte-Carlo estimator of the gradient w.r.t. $\pb$ is known to exhibit high variance. To reduce the variance, we apply the reparameterization trick~\citep{kingma2013auto} and rewrite $\vb = \pb + \wb$ with $\wb \sim \Ncal(\wb| {\bf 0}, \sigma^2)$. This leads to the following gradient of $W_1$,
\begin{align}
\frac{d W_1}{d \pb}  &= -\frac{1}{\beta}\frac{d}{d \pb}\log Z_\beta(S_i, \pb) = \int Q^{\beta}_i(\wb; S_i) \frac{\partial \hat{L}(\pb + \wb, S_i)}{\partial \pb} d\wb, \label{eq:grad_w1} \\
&\text{where,} \;\; Q^{\beta}_i(\wb; S_i) \propto \Ncal(\wb| {\bf 0}, \sigma^2)\exp(-\beta \hat{L}(\pb + \wb, S_i)).\nonumber
\end{align}
As for $W_2$, we also need to evaluate the gradient of $\hat{L}^{\Delta}_{\frac{\alpha}{\beta}}(Q^{\alpha}_i, S_i, S_i')$, where
{\small  \begin{align}
  \frac{d}{d \pb}\hat{L}^{\Delta}_{\frac{\alpha}{\beta}}(Q^{\alpha}_i, S_i, S_i')
    &= \int  Q_i^{\alpha}(\wb; S_i') \frac{\partial \hat{L}^{\Delta}_{\frac{\alpha}{\beta}}(\pb+\wb, S_i, S_i')}{\partial \pb} d\wb + \int \frac{\partial Q_i^{\alpha}(\wb; S'_i)}{\partial \pb}\hat{L}^{\Delta}_{\frac{\alpha}{\beta}}(\pb+\wb, S_i, S_i') d\wb. \label{eq:l_grad}
\end{align}}
The first term of Eq.\eqref{eq:l_grad} is similar to the gradient in Eq.\eqref{eq:grad_w1}.
The Monte-Carlo gradient estimator of the second term, however, exhibits the same high-variance problem as in the policy gradient method. 
As a remedy, we approximate the gradient with the one from the Softmax Policy Gradient~\citep{ding2017cold}, %so that
%\begin{align*}
%    \int \frac{d Q_i^{\alpha}(\wb;S_i')}{d \pb} \hat{L}(\pb + \wb, S_i'') d\wb \simeq \frac{\alpha}{\beta} \int \rbr{Q_i^{\beta}(\wb;S_i) - Q_i^{\alpha}(\wb;S_i')} \frac{d \hat{L}(\pb + \wb, S'_i)}{d \pb} d\wb.
%\end{align*}
which yields a low-variance approximate gradient of $W_2$ (details in Appendix):
{\small\begin{align} 
    \frac{d  W_2}{d \pb} %&\simeq \frac{\alpha}{\beta} \int Q_i^{\beta}(\wb;S_i) \frac{\partial \hat{L}(\pb + \wb; S_i')}{\partial \pb} d\wb + \int Q_i^{\alpha}(\wb;S_i') \frac{\partial \hat{L}^{\Delta}_{\frac{\alpha}{\beta}}(\pb+\wb, S_i, S_i')}{\partial \pb} d\wb \\
    &\simeq \int Q_i^{\alpha}(\wb;S_i') \frac{\partial \hat{L}(\pb + \wb; S_i)}{\partial \pb} d\wb + \frac{\alpha}{\beta} \int \rbr{Q_i^{\beta}(\wb;S_i) - Q_i^{\alpha}(\wb;S_i')} \frac{\partial \hat{L}(\pb + \wb; S_i')}{\partial \pb} d\wb.
    \label{eq:grad_w2}
\end{align}}
The first-term in Eq.\eqref{eq:grad_w2} is similar to the gradient of the First-order MAML (FOMAML, \cite{finn2017model}). The second term involves $Q_i^{\beta}$ and $Q_i^{\alpha}$, which are similar to the leader and the chaser in BMAML~\citep{yoon2018bayesian}. Intuitively, the second term provides additional information that plays a similar role to the high-order derivatives in MAML. However, unlike MAML and BMAML, Eq.\eqref{eq:grad_w2} only involves partial derivatives over $\pb$ (since $\wb$ is not a function of $\pb$) and therefore relies only on first-order derivatives which contribute to its efficiency and stability.

To estimate Eq.\eqref{eq:grad_w1} and Eq.\eqref{eq:grad_w2} in practice, we first draw samples $\wb_{(n)}^{\alpha} \sim Q_i^{\alpha}(\wb;S_i')$ and $\wb_{(n)}^{\beta} \sim Q_i^{\beta}(\wb;S_i)$ using the Monte-Carlo sampling (e.g. SGLD~\citep{welling2011bayesian} or SVGD~\citep{liu2016stein}). After plugging the samples into $\hat{L}(\pb + \wb; S_i)$ and $\hat{L}(\pb + \wb; S'_i)$, we can apply automatic gradient computations (with Tensorflow~\citep{tensorflow2015-whitepaper} or Pytorch~\citep{NEURIPS2019_9015}) over $\pb$ to get the stochastic gradient estimator of $W_1$ and $W_2$.

\section{Experiments}
\label{sec:experiments}
In this section, we evaluate the two PAC-Bayesian algorithms as they were derived in the previous section: PACOH~\citep{rothfuss2020pacoh} of Eq.\eqref{eq:inst-3} and PACMAML of Eq.\eqref{eq:inst-5}.
We use several few-shot learning benchmarks (both synthetic and real), and compare them against other existing meta-learning algorithms, including MAML~\citep{finn2017model}, Reptile~\citep{nichol1803first}, and BMAML~\citep{yoon2018bayesian}. 
To fairly compare with other meta-learning algorithms that optimize a single model, we consider only the empirical Bayes method for PACOH and PACMAML, in which a single MAP solution of $\Qcal$ is used, instead of Bayesian ensembles of $\Qcal$.
% TODO(tomerl): 
% (0) What hypothesis would we like to confirm here? 
%     that PACMAML is the best algorithm?
%     that the experiment results will verify the theorems?

\subsection{Few-Shot Regression Problem} 
\label{sec:regression}
Our first set of experiments are based on the synthetic regression environment setup from \citep{rothfuss2020pacoh}, where the gradient can be obtained analytically. 
\begin{wrapfigure}{R}{5cm}\vspace{-20pt}
\centering
    \includegraphics[width=0.33\textwidth]{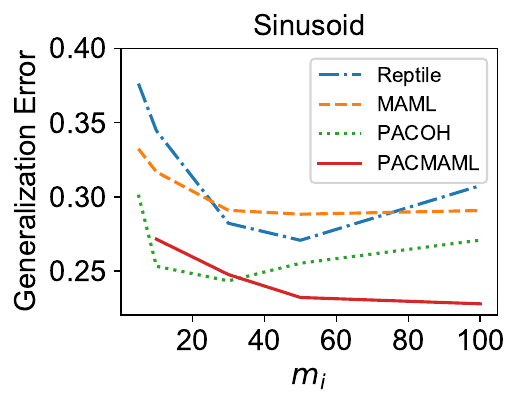}
    \caption{Generalization error (RMSE) on the Sinusoid dataset: 
    PACMAML and MAML continue to improve as $m_i$ increases.}
    \label{fig:sinusoid}
\end{wrapfigure}
% TODO(dingnan): is the order in which everything is presented correct?
The base-learners $Q(S, P)$ are modeled using Gaussian Process (GP) regression with a prior $P_\theta(h) = \mathcal{GP}(h|m_\theta(x),k_\theta(x,x'))$, where the mean function $m_\theta$ and the kernel function $k_\theta$ are instantiated as neural networks as in \citep{rothfuss2020pacoh}.
For every example $z_j = (x_j, y_j)$ and a hypothesis $h$, the loss function is $l(h, z_j) = \| h(x_j) - y_j \|_2^2$ and the empirical risk is $\hat{L}(h, S_i) = \frac{1}{m_i}\sum_{j=1}^{m_i} l(h, z_j)$. 
The hyper-prior $\Pcal(P_\theta) := \Pcal(\theta) = \Ncal(\theta|0, \sigma_0^2I)$ is an isotropic Gaussian defined over the network parameters $\theta$.
 The MAP approximated hyper-posterior takes the form of a delta function, where $\Qcal_{\theta_0}(P_\theta) := \Qcal_{\theta_0}(\theta)=\delta(\theta=\theta_0)$. 
As a result, we have that
%$D_{KL}(\Qcal_{\theta_0} \| \Pcal) = \frac{\|\theta_0\|^2}{2 \sigma_0^2}$
$D_{KL}(\Qcal_{\theta_0} \| \Pcal) = \|\theta_0\|^2/2 \sigma_0^2$, where we use $\sigma_0^2 = 3$ in our experiments. 

We experiment with the synthetic \textit{Sinusoid} environment (details in Appendix \ref{appsubsec:details}), where we fix the number of observed tasks $n = 20$, and vary the number of examples per observed tasks from $m_i \in \cbr{5, 10, 30, 50, 100}$. 
The number of training examples for each target task is fixed to be $m=5$, and another 100 examples for each target task are used as a test set to evaluate the generalization error. 
We report the averaged generalization error over 40 models, with the hyperparameters selected by 4-fold cross-validation over the 20 target tasks. 
Each model is trained on 1 of the 8 pre-sampled meta-training sets (each containing $n=20$ observed tasks) and each set is run with 5 random seeds for network initialization. $\alpha$ and $\beta$ are chosen based on the cross-validation from the grid $\beta/ m_i \in \cbr{10,30,100}$, and $\alpha/m_i \in \cbr{10,20,30,40,50,60}$.

Figure \ref{fig:sinusoid} shows the averaged generalization errors (RMSE) as $m_i$ changes, for the Reptile (with optimal $\qb_i^*$), MAML, PACOH, and PACMAML algorithms. The size of $S_i'$ used for base-learner training in MAML and PACMAML is $m_i'=5$ for all $m_i$. 
The hyperparameter values, the validation errors and the standard errors are reported in the Appendix \ref{appsubsec:additional-result}.
%In the SwissFEL dataset, we fix the number of observed tasks $n = 5$, and vary the number of examples per observed tasks from $m_i \in \cbr{10, 20, 50, 100, 200}$. The number of training examples for each target task is $m=10$. The result is averaged over 5 random seeds. The hyper-parameters are chosen based on a grid search (see appendix). 
As can be seen from the figure, the generalization errors of Reptile (blue) and PACOH (green), both derived from Theorem \ref{thm:pac-meta2}, have a U-shaped curve. 
That is, increasing the meta-training data $m_i$ initially improves generalization in the few-shot target tasks, however, as $m_i$ continues to grows well beyond $m$, generalization suffers. 
This confirms our conjecture from Theorem \ref{thm:pac-meta2}, that larger $m_i$ has a mixed effect on its generalization behavior due to the penalty term $\Delta_\lambda$.
In contrast, the generalization error of MAML and PACMAML, both derived from Theorem \ref{thm:pac-meta3}, is monotonically decreasing as desired.
%This reaffirms our findings in Fig.\ref{fig:sinusoid_bound}, that for PACMAML, increasing $m_i$ is always beneficial. 
Both the generalization error and the bound (in Fig.\ref{fig:sinusoid_bound}) demonstrate that PACMAML is the most effective strategy of utilizing larger meta-training data for few-shot learning.

\subsection{Few-shot Classification Problems}
In addition to the regression problems where the gradients have closed-form, our next experiments evaluate how PACMAML perform on classification tasks using softmax losses with gradient estimators from Eq.\eqref{eq:grad_w1} and Eq.\eqref{eq:grad_w2}. In order to fairly compare with MAML, which has only one set of inner adaptive parameters, we also only use one sample for approximating the inner posterior distribution $\Qcal_i^\alpha$ and $\Qcal_i^\beta$. 
% TODO(tomerl): is this handicapping our algorithm? did we try with more? can MAML try with more? (wait for rebuttal)
\paragraph{Image Classification}
Our first classification experiment is based on the miniImagenet classification task \citep{vinyals2016matching} involving a task adaptation of 5-way classification with a single training example per class (1-shot).
The dataset consists of 60,000 color images of 84×84 dimension. The examples consist of total 100 classes that are partitioned into 64, 12, and 24 classes for meta-train, meta-validation, and meta-test, respectively. We generated the tasks following the same procedure as in \citep{finn2017model} and used the same feature extraction model which contains 4 convolutional layers. Although the original MAML adapted the entire network in the inner loop, \citep{raghu2019rapid} showed similar results by adapting only the top layer, which significantly reduced computational complexity. 
We follow the same "almost no inner loop" (ANIL) setting as \citep{raghu2019rapid}, and compare MAML with BMAML, PACOH and PACMAML. 
Reptile is not included, because it requires full model adaptation.
%For all algorithms, we ran Adagrad for 6-steps in the inner loop to obtain the inner adaptive parameter or a posterior sample $\wb$. The data size of the observed tasks are fixed to be $m_i=16 \times 5 = 80$ following \citep{finn2017model} and $m_i' = m = 5$ (one shot for each of 5 classes). We fixed $\alpha / \beta = m_i' / m_i$ and perform grid search on $\alpha$ as well as the meta and inner learning rate. Other hyperparameters followed the same as \citep{finn2017model}. Details are reported in the appendix. 

For all algorithms, we optimize for 6 steps in the inner loop to obtain the inner adaptive parameter (or a posterior sample $\wb$).
The data sizes of the observed tasks are varied from $m_i=\cbr{10, 20, 40, 80}$ and $m_i' = m = 5$ (one shot for each of 5 classes). We fixed $\alpha / \beta = m_i' / m_i$ and perform grid search on $\alpha$ as well as the meta and inner learning rate on the meta-validation dataset. Other hyperparameters followed the setting in \citep{finn2017model}. Further details are reported in the Appendix. 
\begin{table}[!h]
\centering
  \begin{tabular}{c | c c c c c c c }
     & FOMAML & MAML & BMAML & PACOH & PACMAML \\\hline
%$m_i=5$ & - & - & - & 29.8 $\pm$ 0.8 & - \\
$m_i=10$ & 41.8 $\pm$ 0.9 & 47.3 $\pm$ 0.9 & 29.9 $\pm$ 0.9 & 31.2 $\pm$ 0.8 & {\bf 47.8 $\pm$ 0.9} \\
$m_i=20$ & 44.3 $\pm$ 0.9 & 48.0 $\pm$ 0.9 & 34.3 $\pm$ 0.9 & 37.0 $\pm$ 0.9 & {\bf 49.1 $\pm$ 0.9} \\
$m_i=40$ & 46.2 $\pm$ 1.0 & 47.8 $\pm$ 0.9 & 41.5 $\pm$ 0.9 & 41.6 $\pm$ 0.9 & {\bf 48.9 $\pm$ 0.9} \\
$m_i=80$ & 45.7 $\pm$ 0.9 & 48.1 $\pm$ 0.9 & 44.2 $\pm$ 0.9 & 44.6 $\pm$ 0.9 & {\bf 50.1 $\pm$ 0.9} \\
  \end{tabular}
  \caption{Averaged test accuracy and standard error in the ANIL setting.}
  \label{tab:mini}
 \end{table}\vspace{-10pt}

The main meta-testing results are presented in Table \ref{tab:mini}. 
We find that PACOH performs significantly worse than PACMAML.
One explanation for this is that in PACOH the base-learner (for top layer) is exposed to all $S$ data, and may have already overfit on $S$ and the meta-learner (for lower layers) is unable to learn further. The overfitting of the base-learner is more severe when $m_i$ is small. Surprisingly, we find that BMAML behaves similarly poor in the ANIL 1-particle setting. In FOMAML, MAML and PACMAML, the base-learner is only trained on $S'$ and the meta-learner can learn from the unseen examples in $S$ and therefore no overfitting happens. Both MAML and PACMAML performs significantly better than FOMAML when $m_i$ is small, but their performances saturate and improve little for larger $m_i$, which may due to the domain change between meta-training and testing (as the image class changes).
Overall, PACMAML as a first-order method not only significantly outperforms FOMAML, but also marginally outperforms the high-order MAML, which demonstrates the effectiveness of PACMAML and its gradient estimator.
% TODO(dingnan): verify with more examples that the U-shape is observed.

\paragraph{Natural Language Inference}
Lastly, we evaluate the meta-learning algorithms on the large-scale BERT-base~\citep{devlin2019bert} model containing 110M parameters. Our experiment involves 12 practical natural language inference tasks from~\citep{bansal2019learning} which include:\footnote{Data available at: \url{https://github.com/iesl/leopard}.} (1) entity typing: CoNLL-2003, MIT-Restaurant; (2) rating classification: the review
ratings from the Amazon Reviews dataset in the domain of Books, DVD, Electronics, Kitchen; (3) text classification: social-media datasets from crowdflower that include Airline, Disaster, Emotion, Political Bias, Political Audience, Political Message. 

Following \citep{bansal2019learning}, we used the pretrained BERT-base model as our base model (hyper-prior), and used GLUE benchmark tasks~\citep{wang2018glue} for meta-training the models and meta-validation for hyperparameter search, before fine-tuning them for the 12 target tasks. 
The fine-tuning data contains $k \in \cbr{4, 8, 16}$-shot data for each class in each task. For every $k$, 10 fine-tuning datasets were sampled for each target task. 
The final reported result is the average of the 10 models fine-tuned over these 10 datasets (for each task and each $k$ separately), and evaluated on the entire test set for each target task~\citep{bansal2019learning}. 
The data size of the observed tasks are fixed to be $m_i = 256$, where the data points for each observed task are randomly sampled from the training data of one of the GLUE tasks.
Because the number of classes in these 12 few-shot tasks varies from 2 to 12, we choose the inner data size $m_i'$ from $\cbr{32,64}$ for MAML, BMAML and PACMAML. 
As in \citep{bansal2019learning}, we also partition the set of model parameters to task-specific and task-agnostic. 
For the 12-layer BERT-base model, we consider a hyper-parameter $v \in \cbr{6, 9, 11, 12}$, where only the layers higher than the $v$-th layer are considered task-specific and will be adapted in the inner loop. 
When $v=12$, only the top classification layers are adaptable. For BMAML, PACOH and PACMAML, we performed grid search on $\alpha$ and fixed $\alpha / \beta = m_i' / m_i$.
\begin{table}[!h]
\centering
  \begin{tabular}{c | c c c c c c}
    $k$  & \tiny H-SMLMT~\citep{bansal2020self} & MAML & BMAML & PACOH & PACMAML \\\hline
    4 & 48.61 & 48.21 & 47.27 & 50.47 & \bf 51.58 \\
    8 & 52.92 & 53.52 & 52.08 & 54.83 & \bf 55.68 \\
   16 & 57.90 & 57.38 & 56.53 & 58.22 & \bf 59.18 \\ 
  \end{tabular}
%  \caption{Averaged Generalization error over the 12 NLI tasks.}
%  \label{tab:nli}
% \end{table}
%\begin{table}[!h]\vspace{-.3in}
%\centering
  \begin{tabular}{c | c c c c c }
      & $v$=6 & $v$=9 & $v$=11 & $v$=12 \\\hline
    MAML & 120G & 57G & 16G & 4G \\
    BMAML & 121G & 59G & 19G & 4G \\
    PACMAML & 33G & 16G & 8G & 4G \\
  \end{tabular}
  \caption{Top: Averaged test accuracy over the 12 NLI tasks. Bottom: The comparison of TPU memory (High Bandwidth Memory) usage with different adaptive layer thresholds $v$.}
  \label{tab:nli}
 \end{table}\vspace{-.1in}

Due to space limitation, we only report the averaged generalization error over the 12 tasks in Table \ref{tab:nli} (top). The detailed results of the 12 NLI tasks, their standard errors, as well as the hyperparameter selections are all included in the Appendix. We also include the SOTA results from~\citep{bansal2020self} for comparison and note that PACMAML is consistently the best performer over all three few-shot settings $k=4,8,16$. 
In comparison, MAML and BMAML perform worse, possibly due to sensitivity to learning rates, as suggested by~\citep{bansal2019learning}. 
Beyond generalization errors, in Table \ref{tab:nli} (bottom) we also compare the memory usage of MAML/BMAML against PACMAML over different adaptive layer thresholds $v$. 
These results emphasize the computational advantage of PACMAML by showing that as more layers are adapted (lower $v$), MAML consumes significantly more memory due to its high-order derivatives.

\section{Discussion}
% TODO(tomerl): we should mention
% (1) the problem we address
% (2) the theoretical contributions that are bridging the gap from theory to practice
%   2.1 As opposed to previous bounds, or bounds remain tight in the few shot setting
%   2.2 we provide theoretical justification for existing algorithms, where it was missing (MAML and Reptile)
% (3) experimental results over both synthetic and real data validate that our algorithms outperform the baselines, which we attribute to ??
% (4) future work.

% This paper notes that the current PAC-Bayesian theorems for Meta-Learning are unable to explain the impressive generalization ability of practical algorithms in the Few-Shot setting. To address this gap we extended the theory by formulating two PAC-Bayesian bounds for meta-learning in the practical scenario, where the number of examples in the observed tasks is larger compared to the target tasks.
% As opposed to previous bounds, our bound in thm-4 remains tight under this scenario.
% We then used these bounds to provide theoretical justification for existing meta-learning algorithms (Reptile and MAML) and furthermore, derived PACMAML, a new computationally efficient meta-learning algorithm.
% We conducted experimental results over synthetic data which validate our theoretical findings regarding the tightness of our bounds, and over real data (both vision and language) which show the PACMAML is a more effective meta-learning algorithm compared to other baseline algorithms.

We studied two PAC-Bayesian bounds for meta-learning in the few-shot case, where the number of examples in the target task is significantly smaller than that in the observed tasks.
As opposed to previous bounds, our bound in Theorem \ref{thm:pac-meta3} remains tight in this scenario.
We instantiated these new bounds and related them to the Reptile and MAML algorithms and furthermore derived the PACMAML algorithm, and showed its efficiency and effectiveness over several meta-learning benchmarks.
Broadly speaking, our work falls into the category of PAC-Bayesian theories of non-i.i.d data~\citep{pentina2015lifelong}; however, our study case is more specific and our bounds are based on practical strategies. 
One major limitation of the work is that we do not take into account a data domain shift (e.g.~\citep{germain2016new}), which is often present in practice. 
However, the study of domain shift from a theoretical perspective requires additional assumptions and knowledge about the target data, which do not always exist in practice. 
We leave a deeper discussion and exploration on these topics to future work.

\bibliography{main}
\bibliographystyle{abbrv}

\newpage
\appendix

\section{Proofs}
In this section, we provide proofs of the main theorems presented in the paper. We also provide a brief overview of the proof of Theorem \ref{thm:pac-meta1-new} from \citep{pentina2015lifelong,rothfuss2020pacoh}, since the bound decomposition strategy will also be used in the new theorems of the paper.

\subsection{Brief Proof of Theorem \ref{thm:pac-meta1-new} \cite{pentina2015lifelong,rothfuss2020pacoh}}
\label{sec:proof:thm5}
Given a task environment $T$ and a set of $n$ observed tasks $(D_i, m_i) \sim T$, let $\Pcal$ be a fixed hyper-prior and $\lambda > 0$, $\beta > 0$, with probability at least $1-\delta$ over samples $S_1 \in D_1^{m_1}, \ldots, S_n \in D_n^{m_n}$, we have for all base learner $Q$ and all hyper-posterior $\Qcal$,
\begin{align*}
    R(\Qcal, T) \le &\hat{R}(\Qcal, S_{i=1}^n) + \tilde{\xi} D_{KL}(\Qcal \| \Pcal) \nonumber \\
    &+\frac{1}{n\beta}\sum_{i=1}^n \EE_{P \sim \Qcal} \sbr{D_{KL}(Q(S_i, P) \| P)} + C(\delta, \lambda, \beta, n, m_i),
\end{align*}
where $\tilde{\xi} =\frac{1}{\lambda} + \frac{1}{n\beta}$.

\begin{proof}
The bound in Theorem \ref{thm:pac-meta1-new} was proved by decomposing it into two components: 
\begin{itemize}
\item "Task specific generalization bound", that bounds the generalization error averaged over all observed tasks $\tau_i$:
\begin{align}
    &\EE_{P\sim \Qcal}[\frac{1}{n}\sum_{i=1}^n L(Q(S_i, P), D_i)] \nonumber\\ 
    \le& \hat{R}(\Qcal, S_{i=1}^n) + \frac{1}{n\beta} D_{KL} (\Qcal \| \Pcal) + \frac{1}{n\beta} \sum_{i=1}^n \EE_{P \sim \Qcal} \sbr{D_{KL}(Q(S_i, P) \| P)}\nonumber\\
    &+ \frac{1}{n\beta} \log \frac{1}{\delta} + \frac{1}{n}\sum_{i=1}^n \frac{m_i}{\beta}\Psi_1(\frac{\beta}{m_i}) \label{eq:task_specific1_new}
\end{align}
where 
\begin{align*}
\hat{R}(\Qcal, S_{i=1}^n)&=\EE_{P\sim \Qcal}[\frac{1}{n}\sum_{i=1}^n\hat{L}(Q(S_i, P), S_i)], \\  
\Psi_1(\beta) &= \log\EE_{P \sim \Pcal}\EE_{\hb \sim P} \EE_{z_{ij} \sim D_i} \sbr{e^{\beta (\EE_{z_i \sim D_i}[l(h_i, z_i)] - l(h_i, z_{ij}))}}.
\end{align*}

\item "Task environment generalization bound", that bounds the transfer error from the observed tasks to the new target tasks:
\begin{align}
    R(\Qcal, T) \le &\frac{1}{n}\sum_{i=1}^n\EE_{P \sim \Qcal} \sbr{L(Q(S_i, P), D_i) } \nonumber\\
    &+ \frac{1}{\lambda} \rbr{D_{KL}(\Qcal \| \Pcal) + \log\frac{1}{\delta}}+ \frac{n}{\lambda}\Psi_2(\frac{\lambda}{n}).
    \label{eq:task_environment1}
\end{align}
where 
\begin{align*}
    \Psi_2(\lambda) &= \log \EE_{P \sim \Pcal}\EE_{D_i \sim T, S_i \sim D_i^{m_i}} \sbr{e^{\lambda(\EE_{D_i \sim T, S_i \sim D_i^{m_i}}[R_{S_i}(P)] - R_{S_i}(P))}}.
\end{align*}
\end{itemize}
%NOTE(tomerl): some minor changes here:
Detailed proofs of these two generalization bounds can be found in the appendices of \citep{pentina2015lifelong,rothfuss2020pacoh}. Subsequently, combining Eq.\eqref{eq:task_specific1_new} with Eq.\eqref{eq:task_environment1}, it is straightforward to get Eq.\eqref{eq:task_union1_new}, with
\begin{align}
C(\delta, \lambda, \beta, n, m_i) = \tilde{\xi} \log \frac{1}{\delta} + \frac{1}{n}\sum_{i=1}^n \frac{m_i}{\beta}\Psi_1(\frac{\beta}{m_i}) + \frac{n}{\lambda}\Psi_2(\frac{\lambda}{n}). \label{eq:c}
\end{align}
\end{proof}

\subsection{Proof of Theorem \ref{thm:pac-meta2}}
\label{sec:proof:thm7}
For a target task environment $T$ and an observed task environment $\tilde{T}$ where $\EE_{\tilde{T}}[D]=\EE_T[D]$ and $\EE_{\tilde{T}}[m]\ge\EE_T[m]$, let $\Pcal$ be a fixed hyper-prior and $\lambda > 0$, $\beta > 0$, then with probability at least $1-\delta$ over samples $S_1 \in D_1^{m_1}, \ldots, S_n \in D_n^{m_n}$ where $(D_i, m_i) \sim \tilde{T}$, we have, for all base learners $Q$ and hyper-posterior $\Qcal$, 
\begin{align*}
    R(\Qcal, T) \le &\hat{R}(\Qcal, S_{i=1}^n) + \tilde{\xi} D_{KL}(\Qcal \| \Pcal) +\frac{1}{n\beta}\sum_{i=1}^n \EE_{P \sim \Qcal} \sbr{D_{KL}(Q(S_i, P) \| P)}\nonumber\\
    &+ C(\delta, \lambda, \beta, n, m_i) + \Delta_{\lambda}(\Pcal, T, \tilde{T}), 
\end{align*}
where $\Delta_{\lambda}(\Pcal, T, \tilde{T})=\frac{1}{\lambda}\log \EE_{P \in \Pcal} e^{\lambda(R(P, T) - R(P, \tilde{T}))}$, and $\tilde{\xi} =\frac{1}{\lambda} + \frac{1}{n\beta}$.
\begin{proof}
The "task specific generalization bound" has the same form as Eq.\eqref{eq:task_specific1_new}. 

For the "task environment generalization bound", define the "meta-training" generalization error of a given prior $P$ on the observed task $(D_1, m_1), \ldots, (D_n, m_n) \sim \tilde{T}$ as
\begin{align*}
    R_{S_{\tilde{T}}}(P) \triangleq& \frac{1}{n}\sum_{i=1}^n L(Q(S_i, P), D_i)\\
    =&\frac{1}{n}\sum_{i=1}^n \EE_{z_i \sim D_i} \EE_{h_i \sim Q(h_i|P, S_i)} [L(h_i, z_i)],
\end{align*}
where $S_i \sim D_i^{m_i}$ and $S_{\tilde{T}} = \cbr{S_1, \ldots, S_n}$. 
Similarly, the generalization error on the target task environment $T$ is
\begin{align*}
    R(P, T) = \EE_{(D, m) \sim T}\EE_{S \sim D^m}\EE_{z \in D}\EE_{h \sim Q(h|P, S)} [L(h, z)].
\end{align*}
Using the Markov Inequality, with at least $1-\delta$ probability, 
\begin{align*}
   & \EE_{P \sim \Pcal} \sbr{e^{\lambda (R(P, T) - R_{S_{\tilde{T}}}(P))}} \\
   \le& \frac{1}{\delta} \EE_{P \sim \Pcal}\EE_{D_i \sim T, S_i \sim D_i^{m_i}}^{i=1,\ldots,n} \sbr{e^{\lambda(R(P, T) - R_{S_{\tilde{T}}}(P))} }.
\end{align*}
The left-hand side can be lower bounded by, 
\begin{align*}
    &\log \EE_{P \sim \Pcal} \sbr{e^{\lambda (R(P, T) - R_{S_{\tilde{T}}}(P))}}\\
    =& \log \EE_{P \sim \Qcal} \frac{\Pcal(P)}{\Qcal(P)} e^{\lambda(R(P, T) - R_{S_{\tilde{T}}}(P))}\\
    \ge& \EE_{P\sim \Qcal} \log \frac{\Pcal(P)}{\Qcal(P)} + \lambda\EE_{P\sim \Qcal} [R(P, T) - R_{S_{\tilde{T}}}(P)]\\
    =& -D_{KL}(\Qcal \| \Pcal) + \lambda(R(\Qcal, T) - \EE_{P\sim \Qcal} [R_{S_{\tilde{T}}}(P)]).
\end{align*}
The right-hand side is upper bounded by
\begin{align}
    &\log \frac{1}{\delta} \EE_{P \sim \Pcal}\EE_{D_i \sim T, S_i \sim D_i^{m_i}}^{i=1,\ldots,n} \sbr{e^{\lambda(R(P, T) - R_{S_{\tilde{T}}}(P))} } \nonumber\\
    =& \log \frac{1}{\delta} + \log \EE_{P \sim \Pcal}\EE_{D_i \sim T, S_i \sim D_i^{m_i}}^{i=1,\ldots,n} \sbr{e^{\lambda(R(P, T) - R_{S_{\tilde{T}}}(P))}} \nonumber\\
    =& \log \frac{1}{\delta} + \log \EE_{P \sim \Pcal} \sbr{e^{\lambda(R(P, T) - \EE_{S_{\tilde{T}} \sim \tilde{T}}[R_{S_{\tilde{T}}}(P)])}}\nonumber\\
    &+ \log \EE_{P \sim \Pcal}\EE_{D_i \sim T, S_i \sim D_i^{m_i}}^{i=1,\ldots,n} \sbr{e^{\lambda(\EE_{S_{\tilde{T}}}[R_{S_{\tilde{T}}}(P)] - R_{S_{\tilde{T}}}(P))}} \nonumber\\
    \le& \log \frac{1}{\delta} + \log \EE_{P \sim \Pcal} \sbr{e^{\lambda(R(P, T) - \EE_{S_{\tilde{T}}}[R_{S_{\tilde{T}}}(P)])}} + n\Psi_2(\frac{\lambda}{n}), \label{eq:proof3-1}
\end{align}
where,
\begin{align*}
    &\EE_{S_{\tilde{T}}}[R_{S_{\tilde{T}}}(P)]\\
    \triangleq & \EE_{(D_i, m_i) \sim \tilde{T}, S_i \sim D_i^{m_i}}^{i=1,\ldots,n} [R_{S_{\tilde{T}}}(P)]\\
    =& \frac{1}{n} \sum_{i=1}^n \EE_{(D_i, m_i) \sim \tilde{T}}\EE_{S_i \sim D_i^{m_i}}\EE_{z_i \in D_i}\EE_{h_i \sim Q(h_i|P, S_i)} [L(h_i, z_i)]\\
    =& \EE_{(D, m) \sim \tilde{T}}\EE_{S \sim D^{m}}\EE_{z \in D}\EE_{h \sim Q(h|P, S)} [L(h, z)]\\
    =& R(P, \tilde{T}).
\end{align*}
Combining the left-hand and right-hand bounds together, we have with at least probability $1-\delta$,
\begin{align}
    R(\Qcal, T) \le &\frac{1}{n}\sum_{i=1}^n\EE_{P \sim \Qcal} \sbr{L(Q(S_i, P), D_i) } \nonumber\\
    &+ \frac{1}{\lambda} \rbr{D_{KL}(\Qcal \| \Pcal) + \log\frac{1}{\delta} + n\Psi_2(\frac{\lambda}{n})} \nonumber\\
    &+ \frac{1}{\lambda}\log \EE_{P \in \Pcal} e^{\lambda(R(P, T) - R(P, \tilde{T}))}. \label{eq:task_environment2-1}
\end{align}
Lastly, combining Eq.\eqref{eq:task_environment2-1} with Eq.\eqref{eq:task_specific1_new} yields Eq.\eqref{eq:task_union2}.
\end{proof}
Furthermore, from Theorem \ref{thm:pac-meta2}, it is straightforward to obtain the following corollary.
\begin{corollary}
For a target task environment $T$ and an observed task environment $\tilde{T}$ where $\EE_{\tilde{T}}[D]=\EE_T[D]$ and $\EE_{\tilde{T}}[m]\ge\EE_T[m]$, let $\Pcal$ be a fixed hyper-prior and $\lambda > 0$, $\beta > 0$, then with probability at least $1-\delta$ over samples $S_1 \in D_1^{m_1}, \ldots, S_n \in D_n^{m_n}$ where $(D_i, m_i) \sim \tilde{T}$, we have, for all base learners $Q$ and hyper-posterior $\Qcal$, 
\begin{align}
    R(\Qcal, T) \le &\hat{R}(\Qcal, S_{i=1}^n) + \tilde{\xi} D_{KL}(\Qcal \| \Pcal) +\frac{1}{n\beta}\sum_{i=1}^n \EE_{P \sim \Qcal} \sbr{D_{KL}(Q(S_i, P) \| P)}\nonumber\\
    &+ C(\delta, \lambda, \beta, n, m_i) + \Delta_{\lambda}(\Pcal, \Qcal, T, \tilde{T}), \label{eq:task_union2'}
\end{align}
where $\Delta_{\lambda}(\Pcal, \Qcal, T, \tilde{T})=\min \cbr{\frac{1}{\lambda}\log \EE_{P \in \Pcal} e^{\lambda(R(P, T) - R(P, \tilde{T}))}, R(\Qcal, T) - R(\Qcal, \tilde{T})}$, and $\tilde{\xi} =\frac{1}{\lambda} + \frac{1}{n\beta}$.
\end{corollary}
\begin{proof}
Similar to \eqref{eq:task_environment1}, we have
\begin{align*}
    R(\Qcal, \tilde{T}) \le &\frac{1}{n}\sum_{i=1}^n\EE_{P \sim \Qcal} \sbr{L(Q(S_i, P), D_i) } \nonumber\\
    &+ \frac{1}{\lambda} \rbr{D_{KL}(\Qcal \| \Pcal) + \log\frac{1}{\delta} + n\Psi_2(\frac{\lambda}{n})}.
\end{align*}
A simple reorganization of the terms leads to,
\begin{align}
    R(\Qcal, T) \le &\frac{1}{n}\sum_{i=1}^n\EE_{P \sim \Qcal} \sbr{L(Q(S_i, P), D_i) } \nonumber\\
    &+ \frac{1}{\lambda} \rbr{D_{KL}(\Qcal \| \Pcal) + \log\frac{1}{\delta} + n\Psi_2(\frac{\lambda}{n})} + (R(\Qcal, T) - R(\Qcal, \tilde{T})). \label{eq:task_environment2-2}
\end{align}
Combining Eq.\eqref{eq:task_environment2-2} with Eq.\eqref{eq:task_environment2-1} and Eq.\eqref{eq:task_specific1_new} gives the bound in Eq.\eqref{eq:task_union2'}.
\end{proof}
Note that although Eq.\eqref{eq:task_union2'} gives a potentially tighter bound than Eq.\eqref{eq:task_union2}, empirically it makes little difference because $R(\Qcal, T) - R(\Qcal, \tilde{T})$ is inestimable in practice and cannot be directly optimized as a function of $\Qcal$. We will only numerically estimate its value in synthetic datasets in order to estimate the bound. 

\subsection{Proof of Theorem \ref{thm:pac-meta3}}
\label{sec:proof:thm9}
For a target task environment $T$ and an observed task environment $\tilde{T}$ where $\EE_{\tilde{T}}[D]=\EE_T[D]$ and $\EE_{\tilde{T}}[m]\ge\EE_T[m]$, let $\Pcal$ be a fixed hyper-prior and $\lambda > 0$, $\beta > 0$, then with probability at least $1-\delta$ over samples $S_1 \in D_1^{m_1}, \ldots, S_n \in D_n^{m_n}$ where $(D_i, m_i) \sim \tilde{T}$, and subsamples $S_1' \in D_1^{m_1'} \subset S_1, \ldots, S_n' \in D_n^{m_n'} \subset S_n$, where $\EE[m_i'] = \EE_T[m]$, we have, for all base learner $Q$ and all hyper-posterior $\Qcal$,
{\small\begin{align*}
    R(\Qcal, T) \le & \EE_{P \sim \Qcal} \sbr{\frac{1}{n}\sum_{i=1}^n \hat{L}(Q(S_i', P), S_i)} + \tilde{\xi} D_{KL}(\Qcal \| \Pcal) +\frac{1}{n\beta}\sum_{i=1}^n \EE_{P \sim \Qcal} \sbr{D_{KL}(Q(S_i', P) \| P)}\nonumber\\
    &+ C(\delta, \lambda, \beta, n, m_i), 
\end{align*}}
where $\tilde{\xi} =\frac{1}{\lambda} + \frac{1}{n\beta}$.
\begin{proof}
The "task environment generalization bound" is the same as the one in Theorem \ref{thm:pac-meta1-new}, because the base-learner in observed and target task have the same task environment $T$. Therefore, we have 
\begin{align}
    R(\Qcal, T) \le &\frac{1}{n}\sum_{i=1}^n\EE_{P \sim \Qcal} \sbr{L(Q(S'_i, P), D_i) } + \frac{1}{\lambda} \rbr{D_{KL}(\Qcal \| \Pcal) + \log\frac{1}{\delta}}+ \frac{n}{\lambda}\Psi_2(\frac{\lambda}{n}).
    \label{eq:task_environment3}
\end{align}

As for the "task-specific generalization bound", define,
\begin{align*}
    \hat{L}(\hb) &= \frac{1}{n}\sum_{i=1}^n \frac{1}{m_i} \sum_{j=1}^{m_i} l(h_i, z_{ij}), \;\;
    L(\hb) = \frac{1}{n}\sum_{i=1}^n \EE_{z_i \sim D_i} l(h_i, z_i),
\end{align*}
where $z_{ij} \in S_i$ which is sampled from $D_i$.
According to the Markov inequality, with at least $1-\delta$ probability, we have
\begin{align*}
    &\EE_{P \sim \Pcal}\EE_{\hb \sim P^n} \sbr{e^{n\beta(L(\hb)-\hat{L}(\hb))}} 
    \le \frac{1}{\delta} \EE_{P \sim \Pcal}\EE_{\hb \sim P^n} \EE_{\Sbb \sim \Db^{\mb}} \sbr{e^{n\beta (L(\hb) - \hat{L}(\hb))}}
\end{align*}
Now take the logarithm of both sides, and transform the expectation over $\Pcal, P$ to $\Qcal, Q$, where we use base-learner $Q(S_i', P)$ with $S'_i \in D_i^{m'_i}$. Then the LHS becomes
\begin{align*}
    &\log \EE_{P \sim \Pcal}\EE_{\hb \sim P^n} \sbr{e^{n\beta(L(\hb)-\hat{L}(\hb))}} \\
    =& \log \EE_{P \sim \Qcal}\EE_{\hb \sim \Qb(\Sbb', P)} [\frac{\Pcal(P)\prod_{i=1}^n P(h_i)}{\Qcal(P) \prod_{i=1}^n Q_i(h_i|S_i', P)} e^{n\beta(L(\hb)-\hat{L}(\hb))}] \\
    \ge& -D_{KL} (\Qcal \| \Pcal) - \sum_{i=1}^n \EE_{P \sim \Qcal} \sbr{D_{KL}(Q(S_i', P) \| P)} \\
    &\quad + \beta\EE_{P\sim \Qcal}[\sum_{i=1}^n L(Q(S_i', P), D_i)]-\beta\EE_{P\sim \Qcal}[\sum_{i=1}^n\hat{L}(Q(S_i', P), S_i)].
\end{align*}
The first equation uses the fact that the hyper-prior $\Pcal$ and hyper-posterior $\Qcal$ as well as the prior $P$ are shared across all $n$ observed tasks. 
The inequality uses Jensen's inequality to move the logarithm inside expectation.

The RHS is
\begin{align*}
    &\log\frac{1}{\delta} + \log\EE_{P \sim \Pcal} \EE_{\hb \sim P^n} \EE_{\Sbb \sim \Db^{\mb}} \sbr{e^{n\beta (L(h) - \hat{L}(h))}} \\
    =&\log\frac{1}{\delta} + \log\EE_{P \sim \Pcal}\EE_{\hb \sim P^n} \prod_{i=1}^n \prod_{j=1}^{m_i} \EE_{z_{ij} \sim D_i} \sbr{e^{\frac{\beta}{m_i} (\EE_{z_i \sim D_i}[l(h_i, z_i)] - l(h_i, z_{ij}))}}\\ 
    %\le& \log\frac{1}{\delta} + \log\EE_{P \sim \Pcal}\EE_{\hb \sim P^n} \prod_{i=1}^n \prod_{j=1}^{m_i} e^{\frac{\beta^2}{8m_i^2}}\\
    %=& \log\frac{1}{\delta} + \sum_{i=1}^n \frac{\beta^2}{8m_i}.
    =& \log\frac{1}{\delta} + \sum_{i=1}^n m_i \Psi_1(\frac{\beta}{m_i}).
\end{align*}
Now, combining the LHS and RHS together, we get that with at least $1-\delta$ probability, 
\begin{align}
    &\EE_{P\sim \Qcal}[\frac{1}{n}\sum_{i=1}^n L(Q(S'_i, P), D_i)] \nonumber\\ 
    \le& \EE_{P \sim \Qcal} \sbr{\frac{1}{n}\sum_{i=1}^n L(Q(S'_i, P), S_i)} + \frac{1}{n\beta} D_{KL} (\Qcal \| \Pcal) + \frac{1}{n\beta} \sum_{i=1}^n \EE_{P \sim \Qcal} \sbr{D_{KL}(Q(S'_i, P) \| P)}\nonumber\\
    &+ \frac{1}{n\beta} \log \frac{1}{\delta} + \frac{1}{n}\sum_{i=1}^n \frac{m_i}{\beta}\Psi_1(\frac{\beta}{m_i}). \label{eq:task_specific3}
\end{align}
Combining Eq.\eqref{eq:task_environment3} with Eq.\eqref{eq:task_specific3} immediately yields Eq.\eqref{eq:task_union3}.
\end{proof}

\section{Derivations of MAML and Reptile}
\label{sec:appendix-dirac}
In this section, we derive a couple of meta-learning algorithms based on the MAP estimation of PAC-Bayesian bounds. To this end, we assume that the distribution families of the hyper-posterior $\Qcal(P)$ and posterior $Q_i(h)$ are from delta functions. In addition, we use the isotrophic Gaussian priors for the hyper-prior $\Pcal(P)$ and the prior $P(h)$ on all model parameters,
\begin{align*}
    \Pcal(P) &= \Ncal(\pb | 0, \sigma_0^2) \\
    \Qcal(P) &= \delta(\pb = \pb_0) \\
    P(h) &= \Ncal(\hb | \pb, \sigma^2) \\
    Q_i(h) &= \delta(\hb = \qb_i) \;\; \forall i=1,\ldots, n.
\end{align*}
This way we have a closed form solution for the two KL terms, which are (up to a constant)
\begin{align*}
    D_{KL}(\Qcal \| \Pcal) =& \int d\pb \delta(\pb = \pb_0) \cdot \rbr{\frac{\|\pb\|^2}{2\sigma_0^2} + \frac{k}{2}\log(2 \pi \sigma_0^2) + \log \delta(\pb=\pb_0)} \\ =&\frac{\|\pb_0\|^2}{2\sigma_0^2} + \frac{k}{2}\log(2 \pi \sigma_0^2) + c,
\end{align*}
where $k$ is the dimension of $\pb$ and $c$ is a constant. Similarly,
\begin{align*}
    &\EE_{P \sim \Qcal}[D_{KL}(Q_i | P)] \\
    =& \int d\pb \delta(\pb = \pb_0) \int d\hb \delta(\hb = \qb_i) \cdot \rbr{\frac{\|\hb-\pb\|^2}{2 \sigma^2} + \frac{k}{2}\log(2 \pi \sigma^2) + \log \delta(\hb = \qb_i)} \\
    =& \int d\pb \delta(\pb = \pb_0) \cdot \frac{\|\pb-\qb_i\|^2}{2\sigma^2} + \frac{k}{2}\log(2 \pi \sigma^2) + c \\
    =& \frac{\|\pb_0-\qb_i\|^2}{2\sigma^2} + \frac{k}{2}\log(2 \pi \sigma^2) + c.
\end{align*}

Plugging in the above results, the PAC-Bayesian bound ($PacB$) in Eq.\eqref{eq:task_union2} and Eq.\eqref{eq:task_union3} are both of the form of,
\begin{align*}
&PacB 
%=&\EE_{P \sim \Qcal} \sbr{\frac{1}{n}\sum_{i=1}^n L(Q(S_i, P), S_i)} + \tilde{\xi} D_{KL}(\Qcal \| \Pcal) +\frac{1}{n\beta}\sum_{i=1}^n \EE_{P \sim \Qcal} \sbr{D_{KL}(Q(S_i, P) \| P)} \\
%    &\quad + \tilde{\xi} \log \frac{2}{\delta} + \frac{n}{\lambda}\Psi(\frac{\lambda}{n}) + \frac{1}{n}\sum_{i=1}^n \frac{m_i}{\beta} \Psi(\frac{\beta}{m_i}) \nonumber\\
=\frac{1}{n} \sum_{i=1}^n  L(\qb_i, S_i) + \frac{\tilde{\xi}\|\pb_0\|^2}{2\sigma_0^2} + \frac{1}{n\beta}\sum_{i=1}^n \frac{\|\pb_0-\qb_i\|^2}{2\sigma^2} + C', 
\end{align*}
where the constant $C'$ corresponding to Eq.\eqref{eq:task_union2} and Eq.\eqref{eq:task_union3} are different by $\Delta_\lambda$. The only free variable of $PacB$ is $\pb_0$. The base-learner $\qb_i$ can be any function of $\pb_0$ and $S_i$ for Eq.\eqref{eq:task_union2} or $S_i'$ for Eq.\eqref{eq:task_union3}. One could find the MAP estimation of $PacB$ by gradient descent with respect to $\pb_0$. 

Note that in Eq.\eqref{eq:task_union2}, for a given $\pb_0$ and $S_i$, there exists an optimal base-learner $\qb_i^*$ in the form of,
\begin{align*}
    \qb_i^* = \argmin_{\qb_i} (PacB) = \argmin_{\qb_i} \sbr{L(\qb_i, S_i) + \frac{\|\pb_0-\qb_i\|^2}{2\beta\sigma^2}}.
\end{align*}

Given the optimal $\qb_i^*$, the full derivative of $PacB$ with respect to $\pb_0$ is substantially simpler,
\begin{align}
\frac{d (PacB)}{d \pb_0} &= \frac{\partial (PacB)}{\partial \pb_0} + \inner{\frac{\partial \qb_i^*}{\partial \pb_0}} {\frac{\partial (PacB)}{\partial \qb_i^*}} \nonumber\\
&= \frac{\partial (PacB)}{\partial \pb_0} = \frac{\tilde{\xi}\pb_0}{ \sigma_0^2} + \frac{1}{n}\sum_{i=1}^n \frac{\pb_0-\qb_i^*}{\beta\sigma^2}, \label{eq:grad_pac_delta1}
\end{align}
where the 2nd equation is because $\frac{\partial (PacB)}{\partial \qb_i^*} = 0$ for the optimal base-leaner $\qb_i^*$. Eq.\eqref{eq:grad_pac_delta1} is the equivalent to the meta-update of the Reptile algorithm \citep{nichol1803first}, except that Reptile does not solve for the optimal base learner $\qb_i^*$.

From the optimal condition, the base-learner $\qb_i^*$ satisfies,
\begin{align*}
    \frac{\pb_0 - \qb_i^*}{\beta \sigma^2} = \nabla_{\qb_i^*} L(\qb_i^*, S_i). 
\end{align*}
Therefore, we can rewrite Eq.\eqref{eq:grad_pac_delta1} in the form of the \emph{implicit gradient},
\begin{align*}
   \frac{d (PacB)}{d \pb_0} = \frac{\tilde{\xi}\pb_0}{ \sigma_0^2} + \frac{1}{n}\sum_{i=1}^n \nabla_{\qb^*_i} L(\qb_i^*, S_i).
\end{align*}

In contrast, the standard multi-task objective uses the \emph{explicit gradient}, where $\qb_i=\pb_0$ and
\begin{align*}
   \frac{d (PacB)}{d \pb_0} =  \frac{\tilde{\xi}\pb_0}{ \sigma_0^2} + \frac{1}{n}\sum_{i=1}^n \nabla_{\pb_0} L(\pb_0, S_i).
\end{align*}

\section{Derivations of PACMAML}
For Theorem \ref{thm:pac-meta3}, we use the following posterior as the base-learner for observed task $\tau_i$, 
\begin{align*}
    Q_i(S_i', P)(h) = \frac{P(h)\exp(-\alpha \hat{L}(h, S_i'))}{Z_\alpha(S_i', P)}. 
\end{align*}
Plugging this $Q_i$ into Eq.\eqref{eq:task_union3}, we have
\begin{align*}
&R(\Qcal, T) \\
\le&  \EE_{P \sim \Qcal} \sbr{\frac{1}{n}\sum_{i=1}^n \hat{L}(Q_i, S_i)} + \tilde{\xi} D_{KL}(\Qcal \| \Pcal) +\frac{1}{n\beta}\sum_{i=1}^n \EE_{P \sim \Qcal} \sbr{D_{KL}(Q_i\| P)} + C  \\
    =& \frac{1}{n}\sum_{i=1}^n  \EE_{P \sim \Qcal} \sbr{\hat{L}(Q_i, S_i) + \frac{1}{\beta} D_{KL}(Q_i \| P)} + \tilde{\xi} D_{KL}(\Qcal \| \Pcal) + C \\  
    =& \frac{1}{n}\sum_{i=1}^n  \EE_{P \sim \Qcal} \EE_{h \sim Q_i} \sbr{ \hat{L}(h, S_i) + \frac{1}{\beta} \log Q_i(h) - \frac{1}{\beta} \log P(h)} + \tilde{\xi} D_{KL}(\Qcal \| \Pcal) + C \\      
    =& \frac{1}{n}\sum_{i=1}^n  \EE_{P \sim \Qcal} \EE_{h \sim Q_i} \sbr{ \hat{L}(h, S_i) - \frac{\alpha}{\beta} \hat{L}(h, S_i') - \frac{1}{\beta} \log Z_\alpha(S_i', P))} + \tilde{\xi} D_{KL}(\Qcal \| \Pcal) + C \\    
    = &\frac{1}{n}\sum_{i=1}^n\EE_{P \sim \Qcal} [-\frac{1}{\beta}  \log Z_\alpha(S_i', P) + \hat{L}(Q_i, S_i) - \frac{\alpha}{\beta}\hat{L}(Q_i, S_i')]
    + \tilde{\xi} D_{KL}(\Qcal \| \Pcal) + C.
\end{align*}
where $C = \tilde{\xi} \log \frac{2}{\delta} + \frac{n}{\lambda}\Psi(\frac{\lambda}{n}) + \frac{1}{n}\sum_{i=1}^n \frac{m_i}{\beta} \Psi(\frac{\beta}{m_i})$.

\subsection{The Gradient Estimator of PACOH and PACMAML}
Assuming that the model hypothesis $h$ is parameterized by $\vb$ such that $\hat{L}(h, S_i) \triangleq \hat{L}(\vb, S_i)$, and $\vb$ has prior $P(\vb) = \Ncal(\vb| \pb, \sigma^2)$ with meta-parameter $\pb$, then
\begin{align*}
\log Z_\beta(S_i, \pb) = \log \int \Ncal(\vb| \pb, \sigma^2)\exp(-\beta \hat{L}(\vb, S_i)) d\vb. %\label{eq:inst-3'}
\end{align*}
Note that the parameter $\pb$ appears in the probability distribution of the expectation, and the naive Monte-Carlo gradient estimator of such gradient is known to exhibit high variance. To reduce the variance, we apply the reparameterization trick~\citep{kingma2013auto} and rewrite $\vb = \pb + \wb$ with $\wb \sim \Ncal(\wb| {\bf 0}, \sigma^2)$, then
\begin{align*}
\log Z_\beta(S_i, \pb) = \log \int \Ncal(\wb| {\bf 0}, \sigma^2)\exp(-\beta \hat{L}(\pb + \wb, S_i)) d\wb. %\label{eq:inst-3'}
\end{align*}
This leads to the gradient of $W_1$ in the following form,
\begin{align*}
\frac{d}{d \pb} W_1 &= -\frac{1}{\beta}\frac{d}{d \pb}\log Z_\beta(S_i, \pb) = \int Q^{\beta}_i(\wb; S_i) \frac{\partial \hat{L}(\pb + \wb, S_i)}{\partial \pb} d\wb, \\
&\text{where,} \;\; Q^{\beta}_i(\wb; S_i) \propto \Ncal(\wb| {\bf 0}, \sigma^2)\exp(-\beta \hat{L}(\pb + \wb, S_i)).\nonumber
\end{align*}
As for $W_2$, the first term is simlar to $W_1$, but we also need to evaluate the gradient of $\hat{L}^{\Delta}_{\frac{\alpha}{\beta}}(Q^{\alpha}_i, S_i, S_i')$, which is
{\small  \begin{align}
  \frac{d}{d \pb}\hat{L}^{\Delta}_{\frac{\alpha}{\beta}}(Q^{\alpha}_i, S_i, S_i')
    &= \int  Q_i^{\alpha}(\wb; S_i') \frac{\partial \hat{L}^{\Delta}_{\frac{\alpha}{\beta}}(\pb+\wb, S_i, S_i')}{\partial \pb} d\wb + \int \frac{\partial Q_i^{\alpha}(\wb; S'_i)}{\partial \pb}\hat{L}^{\Delta}_{\frac{\alpha}{\beta}}(\pb+\wb, S_i, S_i') d\wb. \label{eq:grad_Ldelta}
\end{align}}
The second term of Eq.\eqref{eq:grad_Ldelta} is equivalent to,
\begin{align*}
    &\int \frac{\partial Q_i^{\alpha}(\wb;S_i')}{\partial \pb} \hat{L}^{\Delta}_{\frac{\alpha}{\beta}}(\pb+\wb, S_i, S_i') d\wb \\
    =& -\frac{1}{\beta}\frac{\partial}{\partial \pb}\int Q_i^{\alpha}(\wb;S_i') \text{stop\_grad}\rbr{-\beta\hat{L}^{\Delta}_{\frac{\alpha}{\beta}}(\pb+\wb, S_i, S_i')} d\wb.
\end{align*}
The Monte-Carlo gradient estimator of this has the same high-variance problem as in the policy gradient method, which causes unreliable inference without warm-start. Instead, we apply the cold-start policy gradient method by approximating the loss with the one from the softmax value function~\citep{ding2017cold} as follows,
\begin{align*}
&-\frac{1}{\beta}\int Q_i^{\alpha}(\wb;S_i') \text{stop\_grad}\rbr{-\beta\hat{L}^{\Delta}_{\frac{\alpha}{\beta}}(\pb+\wb, S_i, S_i')} d\wb \\
\ge& -\frac{1}{\beta}\log \int Q_i^{\alpha}(\wb;S_i') \exp\rbr{\text{stop\_grad}\rbr{-\beta\hat{L}^{\Delta}_{\frac{\alpha}{\beta}}(\pb+\wb, S_i, S_i')}} d\wb.
\end{align*}
Then we take the gradient of the softmax value function,
\begin{align*}
    &-\frac{1}{\beta}\frac{\partial}{\partial \pb} \log \int Q_i^{\alpha}(\wb;S_i') \exp\rbr{\text{stop\_grad}\rbr{-\beta\hat{L}^{\Delta}_{\frac{\alpha}{\beta}}(\pb+\wb, S_i, S_i')}} d\wb \\
    =& -\frac{1}{\beta} \frac{\int  \frac{\partial Q_i^{\alpha}(\wb;S_i')}{\partial \pb}  \exp\rbr{\text{stop\_grad}\rbr{-\beta\hat{L}^{\Delta}_{\frac{\alpha}{\beta}}(\pb+\wb, S_i, S_i')}} d\wb}{\int Q_i^{\alpha}(\wb;S_i') \exp\rbr{\text{stop\_grad}\rbr{-\beta\hat{L}^{\Delta}_{\frac{\alpha}{\beta}}(\pb+\wb, S_i, S_i')}} d\wb}  \\
    =& -\frac{1}{\beta} \frac{\int  \frac{\partial \log Q_i^{\alpha}(\wb;S_i')}{\partial \pb} Q_i^{\alpha}(\wb;S_i') \exp\rbr{\text{stop\_grad}\rbr{-\beta\hat{L}^{\Delta}_{\frac{\alpha}{\beta}}(\pb+\wb, S_i, S_i')}} d\wb}{\int Q_i^{\alpha}(\wb;S_i') \exp\rbr{\text{stop\_grad}\rbr{-\beta\hat{L}^{\Delta}_{\frac{\alpha}{\beta}}(\pb+\wb, S_i, S_i')}} d\wb}  \\
    =& -\frac{1}{\beta} \frac{\int  \frac{\partial \log Q_i^{\alpha}(\wb;S_i')}{\partial \pb} \Ncal(\wb| {\bf 0}, \sigma^2)\exp(-\beta \hat{L}(\pb + \wb, S_i)) d\wb}{\int \Ncal(\wb| {\bf 0}, \sigma^2)\exp(-\beta \hat{L}(\pb + \wb, S_i)) d\wb}  \\
    =& -\frac{1}{\beta} \int Q_i^{\beta}(\wb;S_i) \frac{\partial \log Q_i^{\alpha}(\wb;S_i')}{\partial \pb} d \wb\\
    =& \frac{\alpha}{\beta} \int \rbr{Q_i^{\beta}(\wb;S_i) - Q_i^{\alpha}(\wb;S_i')} \frac{\partial \hat{L}(\pb + \wb, S'_i)}{\partial \pb} d\wb.
\end{align*}
This yields the overall gradient of $W_2$ to be,
\begin{align*}
    \frac{d}{d \pb} W_2 \simeq& \frac{\alpha}{\beta}\int Q^{\alpha}_i(\wb; S'_i) \frac{\partial \hat{L}(\pb + \wb, S'_i)}{\partial \pb} d\wb + \int Q_i^{\alpha}(\wb;S_i') \frac{\partial \hat{L}^{\Delta}_{\frac{\alpha}{\beta}}(\pb+\wb, S_i, S_i')}{\partial \pb} d\wb \\
    & \quad + \frac{\alpha}{\beta} \int \rbr{Q_i^{\beta}(\wb;S_i) - Q_i^{\alpha}(\wb;S_i')} \frac{\partial \hat{L}(\pb + \wb, S'_i)}{\partial \pb} d\wb \\
    =&\frac{\alpha}{\beta} \int Q_i^{\beta}(\wb;S_i) \frac{\partial \hat{L}(\pb + \wb; S_i')}{\partial \pb} d\wb + \int Q_i^{\alpha}(\wb;S_i') \frac{\partial \hat{L}^{\Delta}_{\frac{\alpha}{\beta}}(\pb+\wb, S_i, S_i')}{\partial \pb} d\wb \\
    =& \int Q_i^{\alpha}(\wb;S_i') \frac{\partial \hat{L}(\pb + \wb; S_i)}{\partial \pb} d\wb + \frac{\alpha}{\beta} \int \rbr{Q_i^{\beta}(\wb;S_i) - Q_i^{\alpha}(\wb;S_i')} \frac{\partial \hat{L}(\pb + \wb; S_i')}{\partial \pb} d\wb.
\end{align*}

The Pseudocode of PACMAML is shown in Algorithm 1.
\begin{algorithm}
\begin{algorithmic}
\State Input: $\sigma$, $\eta$, $\lambda$, $\alpha$, $\beta$, $N$, $K$.
\State Initialize: $\pb_0$.
\For{$i = 0, \ldots , N-1$}
\State $\wb^\alpha_{i, 0}=0, \wb^\beta_{i, 0}=0$
\For{$k = 0, \ldots, K-1$}
\State $\wb^\alpha_{i, k+1} = \wb^{\alpha}_{i, k} - \eta \rbr{\log \Ncal(\wb_{i, k}| 0, \sigma^2) - \beta \hat{L}(\pb_i + \wb_{i, k}, S'_i)}$
\State $\wb^\beta_{i,k+1} = \wb^{\beta}_{i, k} - \eta \rbr{\log \Ncal(\wb_{i, k}| 0, \sigma^2) - \alpha \hat{L}(\pb_i + \wb_{i, k}, S_i)}$
\EndFor
\State $\pb_{i+1} = \pb_i - \lambda \nabla_{p} \rbr{\hat{L}(\pb_i + \wb^{\alpha}_{i, K}, S_i) - \frac{\alpha}{\beta}\hat{L}(\pb_i + \wb^{\alpha}_{i, K}, S'_i) + \frac{\alpha}{\beta} \hat{L}(\pb_i + \wb^{\beta}_{i, K}, S_i)}$
\EndFor
\State Output: $\pb_N$.
\end{algorithmic}
\label{alg:pacmaml-code}
\caption{Pseudocode of PACMAML with approximate gradient estimation. Every posterior is approximated by 1 sample of SVGD, which reduces to SGD. For notation simplicity, we also assume both inner and outer loop uses a gradient decent with fixed learning rate. }
\end{algorithm}

\section{Experiment Details of the Regression Problem}
\label{sec:r_detail}
\subsection{Gaussian Process Model Details}
We use the Gaussian process prior, where $P_\theta(h) = \mathcal{GP}(h|m_\theta(x),k_\theta(x,x'))$ and $k_{\theta}(x, x') = \frac{1}{2}\exp{(-\|\phi_\theta(x)-\phi_\theta(x')\|^2})$. Both $m_\theta(x)$ and $\phi_\theta(x)$ are instantiated to be neural networks. The networks are composed of an input layer of size $1 \times 32$, a hidden layer of size $32 \times 32$. $m_\theta$ and $\phi_\theta$ has an output layer of size $32 \times 1$ and $32\times 2$, respectively. 

We focused on regression problems where for every example $z_j = (x_j, y_j)$ and a hypothesis $h$, the $l_2$-loss function is used so that $l(h, z_j) = \| h(x_j) - y_j \|_2^2$. This leads to a Gaussian likelihood function. Assuming there are $m$ examples in the dataset, we have
\begin{align*}
    P(y | h, x) =& \Ncal(h, \frac{m}{2\alpha} I)\\
    =& \frac{1}{(\pi m / \alpha)^{m/2}} \exp\rbr{-\frac{\alpha}{m} \sum_{j=1}^m (h(x_j) - y_j)^2}.
\end{align*}
As a result, the partition function $Z_{\alpha}(S, P)$ is,
\begin{align*}
    Z_{\alpha}(S, P) &= (\pi m / \alpha)^{m/2} \int_h dh P(y| h, x) P_\theta(h) \\
    &= (\pi m / \alpha)^{m/2} \Ncal(y | m_{\theta}(x), k_\theta(x,x') + \frac{m}{2\alpha} I),
\end{align*}

We apply the GP base-learner $Q$ on the the observed data $S_i$ of task $\tau_i$. For notation simplicity, let us denote $Q_i(h^i | S_i, P) = \mathcal{N}(\mu_i, K_i)$, where $h^i$ denotes the model hypothesis (predictions) of the $m_i$ examples in $S_i$. Then we have,
\begin{align*}
    \hat{L}(Q_i, S_i)
    =& \frac{1}{m_i}\int Q_i(h^i) (y^i - h^i)^\top (y^i - h^i) d h^i \\
    =& \frac{1}{m_i} \rbr{y^{i\top} y^i - 2 \mu_i^\top y^{i} + \mu_i^{\top} \mu_i + \text{tr}(K_i)},
\end{align*}
where $y^i$ denotes the labels of the $m_i$ examples in $S_i$.

The hyper-prior $\Pcal(P_\theta) := \Pcal(\theta) = \Ncal(\theta|0, \sigma_0^2I)$ is an isotropic Gaussian defined over the network parameters $\theta$, where we take $\sigma_0^2 = 3$ in our numerical experiments. The MAP approximated hyper-posterior takes the form of a delta function, where $\Qcal_{\theta_0}(P_\theta) := \Qcal_{\theta_0}(\theta)=\delta(\theta=\theta_0)$. As a result, we have 
\begin{align*}
    &D_{KL}(\Qcal_{\theta_0} \| \Pcal) \\
    =& \int d\theta \delta(\theta=\theta_0) \rbr{\frac{\|\theta\|^2}{2\sigma_0^2} + \frac{k}{2}\log(2 \pi \sigma_0^2) + \log \delta(\theta=\theta_0)} \\ 
    =&\frac{\|\theta_0\|^2}{2\sigma_0^2} + \frac{k}{2}\log(2 \pi \sigma_0^2) + c,
\end{align*}
which combined with $\tilde{\xi}$ becomes the regularizer on the parameters $\theta_0$.

\subsection{Experiment Details}
\label{appsubsec:details}

In the Sinusoid experiment, the number of available examples per observed task $m_i \in \cbr{5, 10, 30, 50, 100}$. Under the setting of PACOH (Theorem \ref{thm:pac-meta2}), for each different $m_i$, we did a grid search on $\beta/m_i \in \cbr{10,30,100}$. Under the setting of PACMAML (Theorem \ref{thm:pac-meta3}), for each different $m_i$, we did a grid search on $\beta/m_i \in \cbr{10,30,100}$ and $\alpha/\beta \in \cbr{0.1,0.2,0.3,0.4,0.5,0.6}$. We use a subsect $S_i'\subset S_i$ with $m_i'=m$ to train the base-learner in PACMAML. For each hyperparameter setting $\beta$ (and $\alpha$), we trained 40 models. Each model is trained on 1 of the 8 pre-sampled meta-training sets (each containing $n=20$ observed tasks) and each set is run with 5 random seeds of network initialization. The ultimate result for each $\beta$ (and $\alpha$) is the averaged result across all models of that setting. The hyperparameters $\tilde{\xi}$ and $\sigma_0^2$ in the hyper-prior ($\Pcal(\theta) = \Ncal(\theta|0, \sigma_0^2I)$) are chosen to be $\tilde{\xi}=1/(n\beta)$ and $\sigma_0^2=3$. To find the optimal model parameter $\theta_0$, we used the ADAM optimizer with learning rate $3 \times 10^{-3}$. The number of tasks per batch is fixed to 5 across all experiments. We run 8000 iterations for each experiment. 

% \paragraph{SwissFEL}
% The number of available examples per observed task $m_i \in \cbr{10, 20, 50, 100, 200}$. For each hyperparameter setting $\beta$ (and $\alpha$), we trained 5 models with different random seeds of network initialization.

The experiments ran in parallel on several 56-core Intel CLX processors and each experiment runs on a single core. Each iteration in the PACOH and PACMAML setting takes about 0.03-0.06s and 0.07-0.14s to run, respectively, with the exact run-time varying for different number of tasks $n$ and number of examples $m_i$.
\subsection{Additional Results}
\label{appsubsec:additional-result}
We performed the 4-fold cross validation over the 20 target tasks to determine the optimal $\beta$ for PACOH (Theorem \ref{thm:pac-meta2}) or the optimal $\alpha$ and $\beta$ for PACMAML (Theorem \ref{thm:pac-meta3}). For the selected $\alpha$ and $\beta$ form validation, we report the lowest test error the corresponding models can achieve. The results are plotted in Figure \ref{fig:sinusoid-crossval}. For each setting, both the validation and test errors show the same trend, where the error with PACOH setting saturates earlier than that with PACMAML setting.

\begin{figure}[!h]
    \centering
    \includegraphics[width=0.33\textwidth]{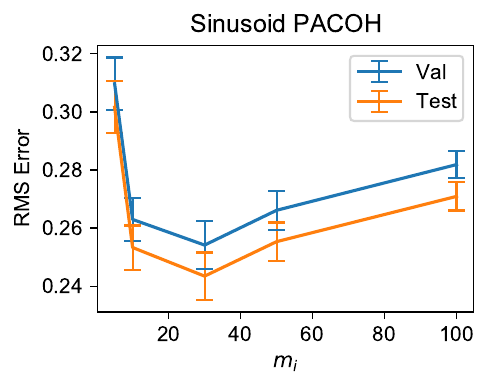}
    \includegraphics[width=0.33\textwidth]{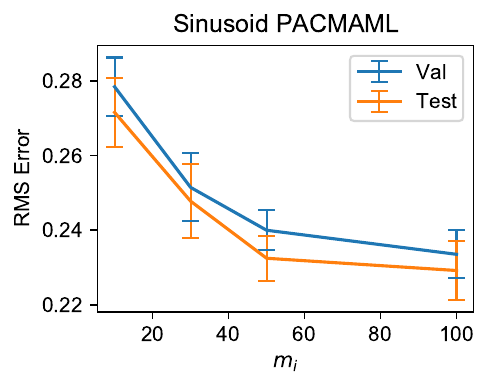}
    \caption{The validation and test error (error bars corresponding to standard errors) on the Sinusoid dataset under the settings of PACOH and PACMAML.}
    \label{fig:sinusoid-crossval}
\end{figure}

\begin{figure}[!h]
    \centering
    \includegraphics[width=0.32\textwidth]{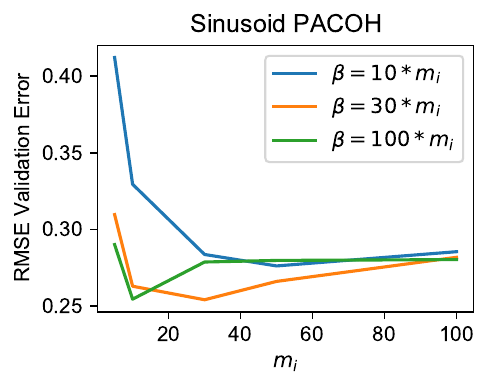}
    \includegraphics[width=0.32\textwidth]{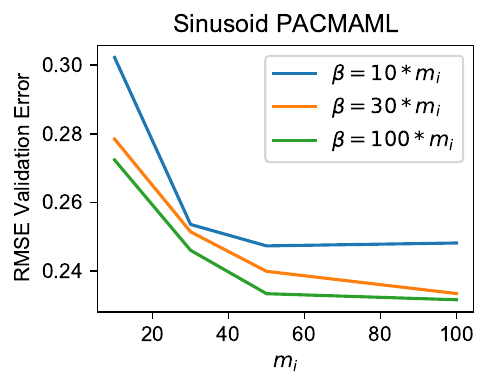}
    \includegraphics[width=0.32\textwidth]{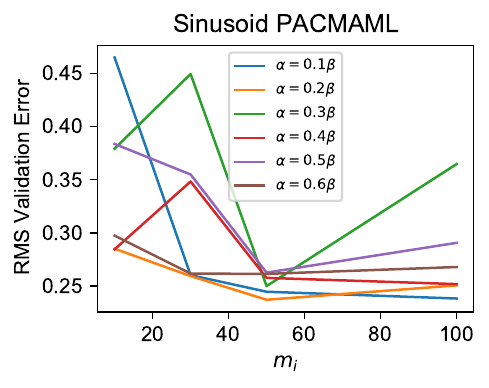}
    \caption{Left: $\beta$-dependence of the RMSE validation error under the PACOH (Theorem \ref{thm:pac-meta2}) setting. Middle and Right: $\beta$- and $\alpha$-dependence of the RMSE validation error under the PACMAML (Theorem \ref{thm:pac-meta3}) setting. $\alpha$ is chosen as the optimal $\alpha$ in the middle plot. $\beta=30*m_i$ in the right plot.}
    \label{fig:parameter-dependence}
\end{figure}

\begin{table}[!h]
\centering
  \begin{tabular}{l | l l l l l l }
    $m_i$  & 5 & 10 & 30 & 50 & 100\\\hline
    $\beta/m_i$ & 100 & 100 & 30 & 30 & 100\\
    % SwissFEL & - & 50 & 200 & 250 & 500 & 1000 
  \end{tabular}
  \caption{Optimal $\beta$ under the setting of PACOH, based on the results of a 4-fold cross validation.}
  \label{tab:thm7-beta}
 \end{table}
 
\begin{table}[!h]
\centering
  \begin{tabular}{l | l l l l l }
    $m_i$ & 10 & 30 & 50 & 100 \\\hline
     $\alpha/\beta$ &  0.2 &  0.2 &  0.2 &  0.1 \\
     $\beta/m_i$ &  100 &  100 &  100 &  100 
  \end{tabular}
  \caption{Optimal $\alpha $ and $ \beta$ values under the setting of PACMAML, based on the results of a 4-fold cross validation.}
  \label{tab:thm9-alpha-beta}
 \end{table}
 
In Table \ref{tab:thm7-beta} and Table \ref{tab:thm9-alpha-beta}, we provide the optimal $\beta$ (and $\alpha$) for PACOH and PACMAML, respectively. In Fig.~\ref{fig:parameter-dependence}, we plotted the validation error for three different values of $\beta$ we used. We see that for both PACOH and PACMAML, the error is large for a small $\beta/m_i=10$. The error with $\beta/m_i=30$ and $\beta/m_i=100$ are similar for PACOH. For PACMAML, the error with $\beta/m_i = 100$ is slightly and consistently better than the error with $\beta/m_i=30$. From the right figure of Fig.~\ref{fig:parameter-dependence} we see that for PACMAML, given $\beta/m_i=30$, $\alpha/\beta$ around $0.2$ achieves lowest validation error.

\subsection{Generalization Bound of PACMAML}
\label{appsubsec:bound}

When $\beta/m_i$ is held as a constant, the $\Psi_1$ and $\Psi_2$ terms of $C(\delta, \lambda, \beta, n, m_i)$ in Eq.\eqref{eq:c} becomes the same across all $m_i$ and both PACOH (Eq. \eqref{eq:inst-3}) and PACMAML (Eq. \eqref{eq:inst-5}). Thus, we exclulde the $\Psi_1$ and $\Psi_2$ terms when comparing the bound values for different $m_i$ and different setups PACOH and PACMAML. In Fig. \ref{fig:pacmaml1-bound} and \ref{fig:pacmaml2-bound} we show the value of each term and the total bound for PACOH and PACMAML obtained from the same set of experiments for Fig. \ref{fig:sinusoid}-\ref{fig:parameter-dependence}. For both PACOH and PACMAML, all three terms $W$, $\tilde{\xi}D_{KL}$ and $\tilde{\xi}\log(1/\delta)$ tend to decrease with larger $m_i$. For PACOH, with the extra term $\Delta_\lambda$ that panalizes larger $m_i$, the total bound either always increases with $m_i$ or first increases then saturates. For PACMAML, without the $\Delta_\lambda$ term, the total bound $W_2+\tilde{\xi}D_{KL}+\tilde{\xi}\log(1/\delta)$ monotonically decreases vs. $m_i$. %Note that for the Sinusoid experiment, we found that $R(\Qcal, T) - R(\Qcal, \tilde{T})$ is lower than $\Delta_\lambda(\Pcal, T, \tilde{T})$. Thus, we write $\Delta_\lambda$ as $\Delta$ here without the $\lambda$ dependence.

In Fig. \ref{fig:bound-comparison}, we show the comparison between the total bound of PACOH and PACMAML. We see that for all $m_i>5$, PACMAML has lower bound for all choices of $\beta$. 

\begin{figure}[!h]
    \centering
    \includegraphics[width=0.32\textwidth]{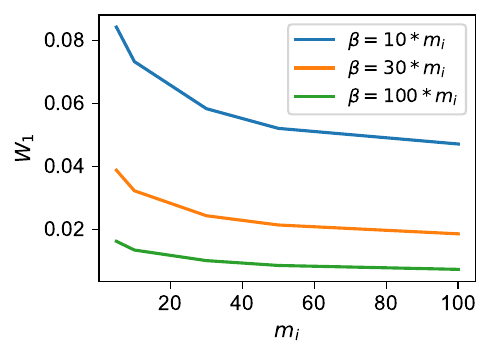}
    \includegraphics[width=0.33\textwidth]{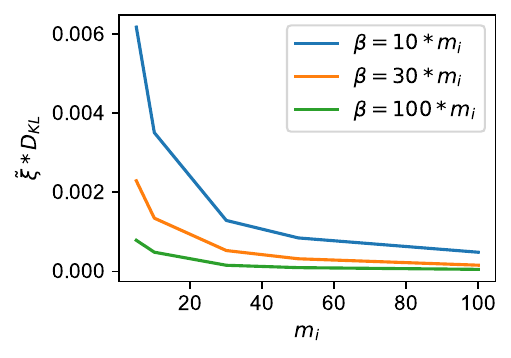}
    \includegraphics[width=0.32\textwidth]{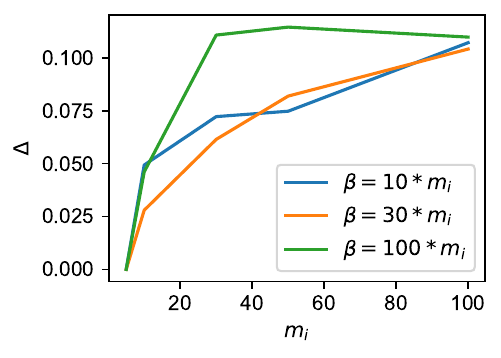}
    \includegraphics[width=0.34\textwidth]{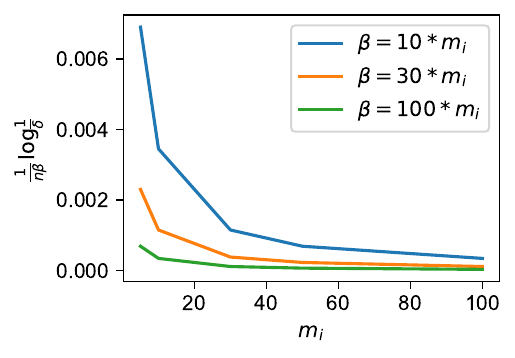}
    \includegraphics[width=0.32\textwidth]{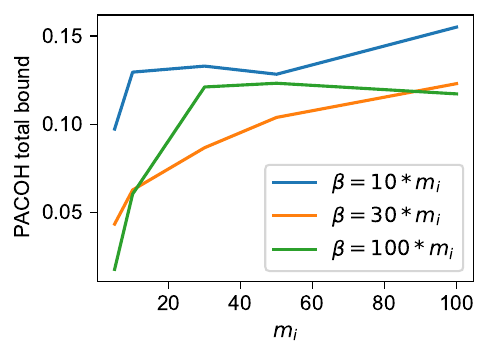}
    \caption{Values of $W_1$, $\tilde{\xi}D_{KL}$ and $\Delta$ terms in the PACOH bound and the total value of the bound for $\beta/m_i\in\cbr{10,30,100}$.}
    \label{fig:pacmaml1-bound}
\end{figure}

\begin{figure}[!h]
    \centering
    \includegraphics[width=0.42\textwidth]{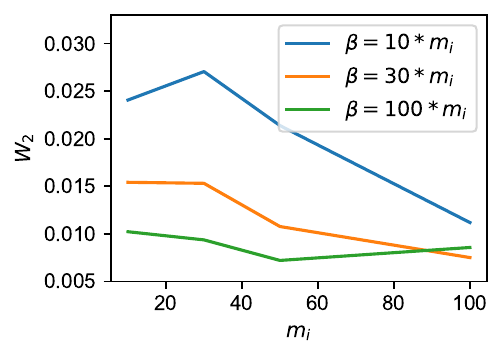}
    \includegraphics[width=0.43\textwidth]{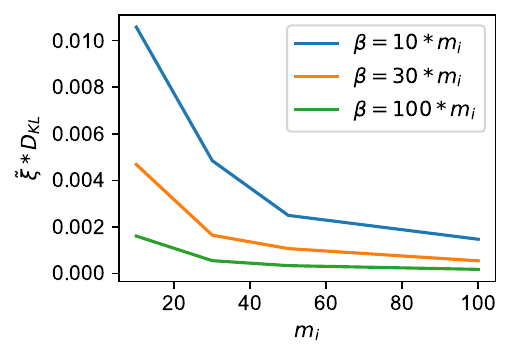}
    \includegraphics[width=0.44\textwidth]{pacmaml1_bound_logdelta.pdf}
    \includegraphics[width=0.43\textwidth]{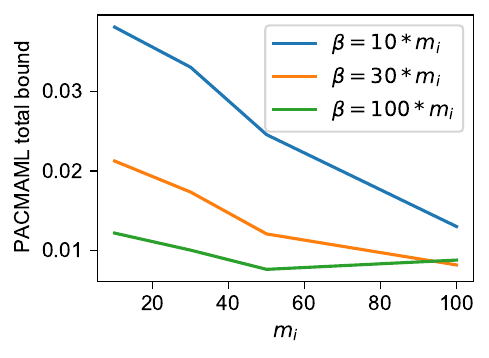}
    \caption{Values of $W_2$ and  $\tilde{\xi}D_{KL}$ terms in the PACMAML bound and the total value of the bound for $\beta/m_i\in\cbr{10,30,100}$. $\alpha$ for each $m_i$ is set to the optimal value according to Fig. \ref{fig:parameter-dependence}.}
    \label{fig:pacmaml2-bound}
\end{figure}

\begin{figure}[!h]
    \centering
    \includegraphics[width=0.44\textwidth]{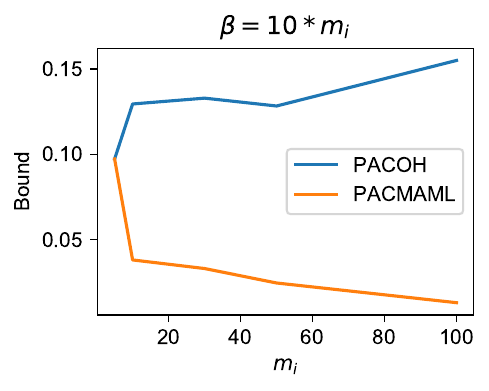}
    \includegraphics[width=0.44\textwidth]{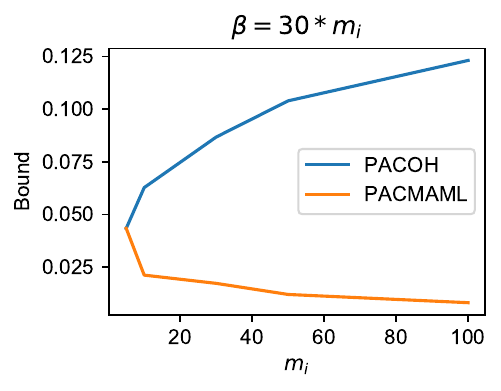}
    \includegraphics[width=0.44\textwidth]{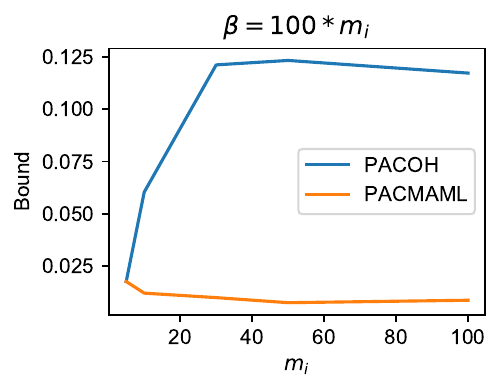}
    \caption{Comparison of the values of PACOH and PACMAML bound for $\beta/m_i\in\cbr{10,30,100}$. $\alpha$ for each $m_i$ for PACMAML is set to the optimal value according to Fig. \ref{fig:parameter-dependence}.}
    \label{fig:bound-comparison}
\end{figure}

\subsection{Experiment for Reptile and MAML}

We also experimented with meta-learning algorithms that use Dirac-measure base-learners, by implementing the Reptile (with optimal $\qb^*$) and the MAML algorithms following the equations of Section \ref{sec:relation}. 

Reptile follows the same experiment setting as PACOH. MAML follows the same experiment setting as PACMAML where $S_i'\subset S_i$, $m_i'=m$. In order to compute the optimal $\qb_i^*$ for Reptile, we use an L-BFGS optimizer in the inner loop with \texttt{lr = 5e-3, history\_size = 10, max\_iter =10}. Other experiment setting and hyperparameter selection procedure are the same as those in Section \ref{appsubsec:additional-result}.

The results of the 4-fold cross validation are plotted in Fig. \ref{fig:maml}. 
The errors of Reptile and MAML follow a very similar trend to the ones with non-Dirac measure base-learners under PACOH and PACMAML setting, respectively (Fig. \ref{fig:sinusoid-crossval}). However, the models with non-Dirac measure base-learners appear to have lower generalization errors than the ones with Dirac measure base-learners (i.e. Reptile and MAML). 

\begin{figure}[!h]
    \centering
    \includegraphics[width=0.33\textwidth]{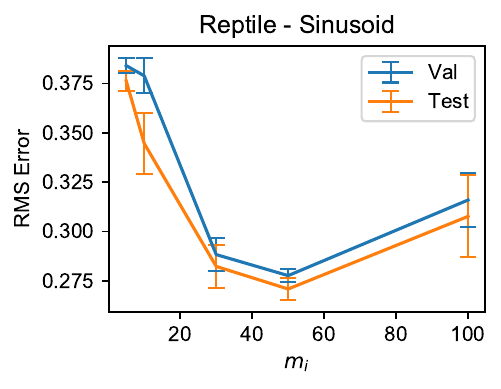}
    \includegraphics[width=0.33\textwidth]{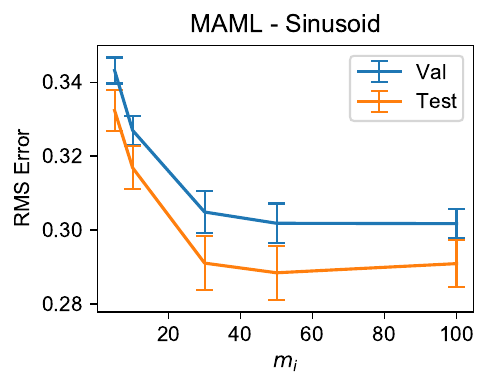}
    \caption{Mean and standard error of the validation   and the test result for Reptile and MAML on Sinusoid. The results are obtained from cross-validation. The error bars in the figures represent the standard errors.}
    \label{fig:maml}
\end{figure}

\section{Experiment Details of Image Classification}
\label{sec:ic_detail}
For most hyperparameters, we followed the same default values as in \citep{finn2017model}. In Table \ref{tab:ic_hyper}, we listed the hyperparameters that we did grid search, and their chosen value based on the meta-validation performance. For the inner learning rate, the search space was $\cbr{0.1, 0.03, 0.001, 0.003}$ for FOMAML, MAML, and PACMAML; the search space was $\cbr{0.1, 0.03, 0.001, 0.003, 0.001, 0.0003, 0.0001}$ for BMAML and PACOH. For the meta-learning rate, we used the default 0.001 for FOMAML, MAML and PACMAML; and searched over $\cbr{0.001, 0.0003, 0.0001, 0.00003}$ for BMAML and PACOH. For $\alpha$, we searched over $\cbr{10, 1.0, 0.1}$ for BMAML, PACOH, PACMAML. We also tried two gradient descent methods in the inner loop: Vanilla GD and ADAGRAD . We found that FOMAML and MAML worked better with Vanilla GD; while BMAML, PACOH and PACMAML worked better with ADAGRAD. $\sigma^2$ was fixed to 1 for PACOH and PACMAML. The number of task per batch was 4 and the network filter size was 64. The total number of meta-training iterations was 60000 for all algorithms. We ran these tasks with 1 NVIDIA P100 GPU per job and each job takes about 2-3 hours to finish.
\begin{table}[!h]
\centering
  \begin{tabular}{c | c | c c c c c}
    $m_i$ & Hyper-parameter & FOMAML & MAML & BMAML & PACOH & PACMAML\\\hline
    %& outer learning rate & - & - & - & 0.0001 & - \\
    %5 & inner learning rate & - & - & - & 0.01 & - \\
    %& $\alpha$ & - & - & - & 10 & - \\ \hline    
    & outer learning rate & 0.001 & 0.001 & 0.0001 & 0.0001 & 0.001 \\
    10 & inner learning rate & 0.1 & 0.1 & 0.003 & 0.01 & 0.03 \\
    & $\alpha$ & - & - & 1.0 & 10 & 1.0 \\ \hline
    & outer learning rate & 0.001 & 0.001 & 0.0001 & 0.0001 & 0.001 \\
    20 & inner learning rate & 0.03 & 0.1 & 0.003 & 0.003 & 0.01 \\
    & $\alpha$ & - & - & 0.1 & 1.0 & 1.0 \\ \hline
    & outer learning rate & 0.001 & 0.001 & 0.0001 & 0.0001 & 0.001 \\
    40 & inner learning rate & 0.03 & 0.03 & 0.003 & 0.003 & 0.01 \\
    & $\alpha$ & - & - & 0.1 & 1.0 & 10 \\ \hline
    & outer learning rate & 0.001 & 0.001 & 0.0001 & 0.0001 & 0.001 \\
    80 & inner learning rate & 0.03 & 0.03 & 0.0003 & 0.003 & 0.01 \\
    & $\alpha$ & - & - & 0.1 & 1.0 & 1.0 \\ \hline
  \end{tabular}
  \caption{The final hyper-parameters of the algorithms in the Mini-imagenet task.}
  \label{tab:ic_hyper}
 \end{table}

\section{Experiment Details of Natural Language Inference}
\label{sec:nli_detail}
We fixed $\sigma^2=0.0004$, which equals to the variance of the BERT parameter initialization. The hyper-parameter $\alpha$ is decided by a grid search over $\cbr{10^2, 10^3,10^4,10^5,10^6,10^7}$ based on the performance on the meta-validation dataset. The inner loop learning rate is $0.001$ for all algorithms. We used 50-step Adagrad optimizer in the inner-loop because it has automatic adaptive learning rate for individual variables which is beneficial for training large models. For the outer-loop optimization, we used the ADAM optimizer with learning rate $10^{-5}$. The final hyperparameters are reported in Table \ref{tab:nli_hyper}. In the few-shot learning phase, we ran the ADAM optimizer for 200 steps with learning rate $10^{-5}$ on the adaptable layers. We ran the tasks with 16 TPUs(v2) per job.

\begin{table}[!h]
\centering
  \begin{tabular}{c | c c c c}
    Hyper-parameter & MAML & BMAML & PACOH & PACMAML\\\hline
    inner learning rate & 0.001 & 0.001 & 0.001 & 0.001 \\
    $v$  & 12 & 12 & 12 & 11 \\
    $m_i'$ & 32 & 64 & 256 & 64  \\
    $m_i$ & 256 & 256 & 256 & 256 \\
    $\alpha$ & - & $10^3$ & $10^4$ & $10^4$ \\
    tasks per batch & 1 & 1 & 1 & 1\\
    meta-training iteration & 10000 & 10000 & 10000 & 10000 \\
  \end{tabular}
  \caption{The final hyper-parameters in the NLI tasks.}
  \label{tab:nli_hyper}
 \end{table}
 
In Table \ref{tab:nli_detail} we report the detailed classification accuracy on the 12 NLI tasks with their standard errors. 
\begin{table}[!h]
\centering
  \begin{tabular}{c | c | c | c c c c}
    Task name & $N$ & $k$ & MAML & BMAML & PACOH & PACMAML\\\hline
     &  & 4 & 63.0$\pm$1.4 & 61$\pm$2.3 &62.1$\pm$2.2 & 68.8$\pm$1.6\\
     CoNLL & 4 & 8 & 74.1$\pm$1.8 & 68$\pm$1.9 &74.9$\pm$1.2 & 79.5$\pm$1.1\\
     &  & 16 & 81.6$\pm$0.6 & 77.9$\pm$1.4 &83$\pm$0.7  & 84.5$\pm$0.6\\ \hline
     &  & 4 & 51.3$\pm$1.8 &47.5$\pm$1.9 &55.9$\pm$1.6  & 60.6$\pm$1\\
    MITR & 8 & 8 & 69.1$\pm$2.1 &64.2$\pm$1.3 &71.8$\pm$0.8 & 70.9$\pm$1\\
     &  & 16 & 78.7$\pm$1.1 &72.2$\pm$1.3 &78.1$\pm$0.6 & 80$\pm$0.6\\ \hline  
     &  & 4 & 60.1$\pm$2.0 &53$\pm$2.7 &60.1$\pm$3.1  & 60.5$\pm$1.9\\
    Airline & 3 & 8 & 64.7$\pm$2.7 &67.4$\pm$2.2 &65$\pm$1.5 & 65.4$\pm$1.7 \\
     &  & 16 & 68.4$\pm$2.2 &66.7$\pm$2.6 &69.6$\pm$1.3 & 69.9$\pm$1.1\\ \hline  
     &  & 4 & 56.3$\pm$0.5 &58.7$\pm$3.1 &58.7$\pm$2.6 & 63.3$\pm$1.3\\
    Disaster & 2 & 8 & 61.5$\pm$0.7 &64.1$\pm$2.3 &64.1$\pm$2.4  & 63.9$\pm$2.9\\
     &  & 16 & 67.7$\pm$0.4 &69.4$\pm$2.0 &71.3$\pm$1.7 & 71.1$\pm$1.6\\ \hline   
     &  & 4 & 13.7$\pm$2.1 &13.9$\pm$0.5 &13.8$\pm$0.5  & 13.7$\pm$0.7\\
    Emotion & 13 & 8 & 15.8$\pm$1.9 &14.6$\pm$1.1 &15$\pm$0.6 & 15.8$\pm$0.6\\
     &  & 16 & 16.7$\pm$0.9 &15.6$\pm$0.7 & 17.2$\pm$0.7 & 16.8$\pm$0.5\\ \hline  
     & & 4 & 58$\pm$2.1 &58$\pm$2.0 &58.8$\pm$2.6  & 59.9$\pm$2.1\\
    Political Bias & 2 & 8 & 60.7$\pm$1.9 &61$\pm$1.9 &62.1$\pm$1.5 & 62$\pm$1.9\\
     & & 16 & 64.6$\pm$0.9 &63.5$\pm$1.2 &63.8$\pm$1.2 & 66$\pm$1\\ \hline  
     & & 4 & 52.2$\pm$0.9 &54.9$\pm$0.7 &53.1$\pm$0.9  & 53.4$\pm$1.3\\
    Political Audience & 2 & 8 & 56.1$\pm$1.5 &55.9$\pm$1.1 &56$\pm$1.3 & 56$\pm$1.2\\
     & & 16 & 56.5$\pm$1.2 &56.9$\pm$1.3 &60$\pm$0.9 & 59.6$\pm$1\\ \hline  
    & & 4 & 18.9$\pm$0.8 &17.4$\pm$0.6 &19.2$\pm$0.7 & 19.3$\pm$0.6\\
    Political Message & 9 & 8 & 22.3$\pm$0.7 &19.3$\pm$0.8 &22.3$\pm$0.6 & 22.6$\pm$0.5\\
     & & 16 & 24.3$\pm$0.8 &21.6$\pm$0.4 &24.9$\pm$0.4 & 25.5$\pm$0.8\\ \hline  
    & & 4 & 58.7$\pm$2.1 &56.2$\pm$2.8 &59$\pm$2.3  & 56.8$\pm$3\\
    Rating Books & 3 & 8 & 61.3$\pm$2.7 &55.1$\pm$2.7 &64.2$\pm$2 & 61.6$\pm$1.5\\
    & & 16 & 62$\pm$1.3 &66.6$\pm$2.1 &63$\pm$2.1 & 60.4$\pm$2.7\\ \hline  
    & & 4 & 49.5$\pm$3.0 &53.7$\pm$2.7 &53.7$\pm$2.1  & 52.4$\pm$1.5\\
    Rating DVD & 3 & 8 & 53.2$\pm$1.6 &51.8$\pm$2.4 & 54.7$\pm$2 & 56$\pm$2\\
     & & 16 & 54.7$\pm$1.2 &57.2$\pm$1.5 & 55.4$\pm$1.3  & 60$\pm$1.4\\ \hline  
    & & 4 & 46.9$\pm$3.1 &44.6$\pm$1.9 & 53.3$\pm$1.7  & 52.4$\pm$2\\
    Rating Electronics & 3 & 8 & 52.5$\pm$1.6 &54.1$\pm$1.6 &55.6$\pm$2 & 56.1$\pm$1.3\\
    & & 16 & 54.7$\pm$1.8 &56.6$\pm$1.8 &57.5$\pm$1.5 & 58.2$\pm$0.7\\ \hline 
     & & 4 & 49.9$\pm$2.4 &48.3$\pm$2.1 &57.9$\pm$1.3 & 57.8$\pm$2 \\
    Rating kitchen & 3 & 8 & 50.9$\pm$2.8 &49.5$\pm$3.1 & 52.3$\pm$2.2 &  58.3$\pm$1.5\\
     & & 16 & 58.7$\pm$1.5 &54.2$\pm$1.8 & 54.8$\pm$1.8  & 58.1$\pm$2.5\\ \hline  
    & & 4 & 48.21  &47.27 & 50.47 &51.58\\
    Overall average & - & 8 & 53.52 &52.08 &54.83 & 55.68\\
    & & 16 & 57.38 &56.53 &58.22 & 59.18 \\ \hline
\end{tabular}
  \caption{Classification accuracy and standard error on the 12 NLI tasks.}
  \label{tab:nli_detail}
 \end{table}

\end{document}